\documentclass[journal]{IEEEtran}

\usepackage{graphicx,amsmath, amssymb,lineno}
\usepackage{subfigure, graphicx}
\usepackage{amsmath,amssymb,amsthm}
\usepackage{algorithmic}
\usepackage{algorithm}
\usepackage{tabularx}
\usepackage{multirow}
\usepackage{makecell}

\newtheorem{prop}{Proposition}

\newtheorem{problem}{Problem}

\begin{document}

\title{ Latent Complete Row Space Recovery for Multi-view Subspace Clustering}

\author{Hong~Tao, Chenping~Hou*,~\IEEEmembership{Member,~IEEE}, Yuhua Qian*, Jubo~Zhu, Dongyun~Yi
\thanks{This work was supported by the NSF of China under Grant 61922087 and Grant 61906201, and the NSF for Distinguished Young Scholars of Hunan Province under Grant 2019JJ20020. Chenping Hou and Yuhua Qian are the corresponding authors.}
\thanks{Hong Tao, Chenping Hou, Jubo Zhu and Dongyun Yi are with the Department of Systems Science, College of Science, National University of Defense Technology, Changsha, 410073, Hunan, China (e-mail: taohong.nudt@hotmail.com, hcpnudt@hotmail.com, ju\_bo\_zhu@aliyun.com, dongyun.yi@gmail.com).}
\thanks{Yuhua Qian is with the Institute of Big Data Science and Industry, Shanxi University, Taiyuan, 030006, Shanxi, China (e-mail: jinchengqyh@126.com).}
}

\markboth{}%
{Tao \MakeLowercase{\textit{et al.}}: Latent Complete Row Space Recovery for Multi-view Subspace Clustering}

\IEEEtitleabstractindextext{
\begin{abstract}
Multi-view subspace clustering has been applied to applications such as image processing and video surveillance, and has attracted increasing attention.
Most existing methods learn view-specific self-representation matrices, and construct a combined affinity matrix from multiple views.
The affinity construction process is time-consuming, and the combined affinity matrix is not guaranteed to reflect the whole true subspace structure.
To overcome these issues, the Latent Complete Row Space Recovery (LCRSR) method is proposed.
Concretely, LCRSR is based on the assumption that the multi-view observations are generated from an underlying latent representation, which is further assumed to collect the authentic samples drawn exactly from multiple subspaces.
LCRSR is able to recover the row space of the latent representation, which not only carries complete information from multiple views but also determines the subspace membership under certain conditions.
LCRSR does not involve the graph construction procedure and is solved with an efficient and convergent algorithm, thereby being more scalable to large-scale datasets.
The effectiveness and efficiency of LCRSR are validated by clustering various kinds of multi-view data and illustrated in the background subtraction task.

\end{abstract}

}

\maketitle

\IEEEdisplaynontitleabstractindextext

\section{Introduction}
In many real-world applications, multi-view data are produced increasingly.
For example, in multi-camera video surveillance, multi-camera networks  record human activities, where each camera corresponds to a view.
For the sake of security, a growing number of multi-camera networks are deployed.
Driven by the requirements to analyzing these multi-view data, multi-view learning was proposed \cite{blum1998cotraining,Bickel2004MvC} and has experienced fast development in recent years \cite{White2012CMSL,xu2013MVLsurvey,Guo2013CSRL,Xu2015Intact,Wang2015consensus,Hou2017MUFE,2017WangLapLRR,Zhang2018gLMSC,Cheng2019Tensor,Wu2019essentialTensor}.
Based on how much label information is used, multi-view learning can be roughly categorized into multi-view supervised learning, multi-view semi-supervised learning and multi-view unsupervised learning.
Since labeling samples is expensive in both time and energy, it is unrealistic to obtain labels for all samples in some applications, e.g., wild face recognition in multi-camera video surveillance.
To overcome this barrier, multi-view unsupervised learning which explores the intrinsic structures of multi-view data using no label information has received increasing attention.

As a typical task in multi-view unsupervised learning, the goal of multi-view subspace clustering is to partition the multi-view instances that are approximately drawn from the same subspace into the same cluster.
Up to present, various multi-view subspace clustering algorithms with promising performance have been developed \cite{Wang2015consensus,zhang2015LTMSC,cao2015DiMSC,Ding2016RMSL,Gao2015MVSC,Luo2018CSMSC,Yin2018SCMV3D,wang2017ECRMSC,Zhang2017LMSC}.
Most of the existing multi-view subspace clustering algorithms are based on Low-Rank Representation (LRR) \cite{Liu2013LRR} or Sparse Subspace Clustering (SSC) \cite{2013SSC}.
Both LRR and SSC construct the affinity matrix by learning the self-representation coefficient matrix of the self-representation model.
The main difference between them is that LRR seeks the lowest-rank self-representation while SSC wants the sparsest one.
Provided that the data points are sufficient and are drawn from independent subspaces, the learned self-representation matrix will be approximately block diagonal and each block corresponds to a cluster  \cite{Lu2019BDR}.
After forming an affinity matrix that encodes the subspace membership, onto which one can obtain the final clustering results by performing standard spectral clustering algorithms, such as Normalized Cut (NCut) \cite{shi2000Ncut}.
The ideas of LRR and SSC are inherited and extended to the multi-view setting.
Existing multi-view subspace clustering approaches learn affinity matrices from multiple views and combine them together to form the final affinity matrix.
The main difference among them is that they learn affinity matrices with distinct models on the relationship of multiple views (refer to the next section for a brief review).
Albeit these methods have appealing performance, most of them learn view-specific self-representations.
Practically, each view only contains partial information about the data \cite{Xu2015Intact,Zhang2017LMSC}.
Thus, learning view-specific self-representations within each view makes data being partially described.
Recognizing this will impair the representation capability of the final affinity matrix, Zhang et al. \cite{Zhang2017LMSC} proposed to learn the self-representation matrix of the underlying latent representation of multiple views.
However, whether the learned latent representation contains complete information from multiple views was not analyzed in \cite{Zhang2017LMSC}.
Thus, the information completeness of the final affinity matrix still remains not guaranteed.
In addition, regarding the time complexity, almost all of existing multi-view subspace clustering methods have a cubic time complexity with respect to (w.r.t.) the number of samples.
This limits their applicability to large-scale problems.

To alleviate the above mentioned problems, in this paper, we propose the Latent Complete Row Space Recovery (LCRSR) method for multi-view subspace clustering.
Concretely, we assume that multiple views are generated from a shared latent representation, which is further assumed to store the samples drawn from a union of multiple subspaces.
Since the row space of data is proved to be able to decide the subspace membership \cite{Liu2013LRR}, under the above assumptions, data samples can be well grouped provided that the row space of the latent representation is recovered.
Meanwhile, considering that the data matrix of each view could be contaminated by gross errors, we aim to simultaneously recover the row space of the latent representation (called latent row space) and the possible errors existed in data's multiple views.
To achieve this goal, we establish the LCRSR model to explore the relationships between data's multiple observations, possible errors and the latent row space.
According to the LCRSR model, the recovered latent row space contains the row spaces of each view as subspaces, thereby carrying complete information about the samples.
To solve the LCRSR model, an iterative alternating algorithm (also referred to as LCRSR) with proved convergence is developed.
When the latent representation assumption completely holds and the true rank of the latent representation is known, we prove that the LCRSR algorithm can restore the latent complete row space exactly.
Once the latent complete row space is restored, the final clustering partition can be obtained by simply applying K-Means onto it \cite{Liu2018RSP}.

Compared with the previous multi-view subspace clustering methods, the proposed LCRSR has two advantages.
On one hand, LCRSR directly recovers the latent complete row space of multi-view data. With the data being sufficiently described by the latent complete representation, LCRSR can achieve better clustering performance.
On the other hand, LCRSR avoids the computationally expensive graph construction process and the subsequent spectral clustering step.
This leads to a less computational complexity, i.e., being quadratic to the number of samples and linear to the dimensionality.
The contributions of this paper are summarized as follows.

\begin{itemize}
  \item Under the assumptions of this paper, the Latent Complete Row Space Recovery model is established to identity the possible errors and restore the row space of the latent complete representation.
  \item Preliminary theoretical analysis about the recovery ability of LCRSR are provided. When the latent representation assumption holds, the latent complete row space can be recovered exactly.
  \item An efficient and convergent algorithm is devised for optimization, reducing the computational complexity with respect to sample size from cubic to quadratic time. Thus, LCRSR is more scalable to large-scale datasets.
  \item The performance of LCRSR is evaluated via subspace recovery on synthetic data, multi-view subspace clustering on various real-world datasets, and background subtraction in multi-camera video surveillance.
      Experimental results demonstrate the effectiveness and efficiency of LCRSR.
\end{itemize}

The remainder of this paper is organized as follows.
Section \ref{nota_sec_RW} introduces the notations and reviews some prevalent multi-view subspace clustering methods.
Section \ref{sec_prob_state} states the problem to be solved and some analysis.
Section \ref{sec_MvRSR} presents the proposed LCRSR method, followed by its optimization in Section \ref{sec_opt}.
The experimental results are displayed in Section \ref{sec_exp}.
Finally, conclusions are made in Section \ref{sec_summary}.

\section{Notations and Related Work}\label{nota_sec_RW}

\subsection{Notations}
In this paper, boldface uppercase/lowercase letters are used to notate matrices/vectors.
Particularly, we use $\mathbf{I}$ and $\mathbf{1}$ to denoted the identity matrix and the vector with all ones.
$m_{ij}$ or $\mathbf{M}_{ij}$ denotes the $(i, j)$-th element of $\mathbf{M}$.
$|\mathbf{M}|$ outputs a matrix with its $(i,j)$-th element being $|m_{ij}|$.
$Tr(\mathbf{M})$ and $\mathbf{M}^T$ are the trace and transpose of $\mathbf{M}$, respectively.
The Frobenius norm of a matrix $\mathbf{M}$ or $\ell_2$ norm of a vector $\mathbf{m}$ is denoted by $\|\mathbf{M}\|$ or $\|\mathbf{m}\|$.
The nuclear norm (the sum of singular values) and the spectral norm (the largest singular value) of $\mathbf{M}$ are denoted by $\|\mathbf{M}\|_*$  and  $\|\mathbf{M}\|_2$, respectively.
The $\ell_1$ and $\ell_{2,1}$ norms of matrices are defined by $\|\mathbf{M}\|_1 = \sum\nolimits_{ij} |m_{ij}|$ and $\|\mathbf{M}\|_{2,1} = \sum\nolimits_j \|\mathbf{M}_{:,j}\|$, where $\mathbf{M}_{:,j}$ denotes the $j$-th column of $\mathbf{M}$.
For presentational convenience, we use $[\mathbf{M}_1, \cdots, \mathbf{M}_k]$ ($[\mathbf{M}_1; \cdots; \mathbf{M}_k]$) to represent the matrix formed by concatenating $\{\mathbf{M}_i\}_{i=1}^k$ along the horizontal (vertical) direction.
For a rank $r$ matrix $\mathbf{M}\in \mathbb{R}^{m\times n}$, the compact Singular Value Decomposition (thin SVD) is defined as $\mathbf{M} = \mathbf{U}\boldsymbol{\Sigma} \mathbf{V}^T = \sum\limits_{i=1}^r \sigma_i \mathbf{U}_{:,i}(\mathbf{V}_{:,i})^T$, where $\mathbf{U}\in \mathbb{R}^{m\times r}$ and $\mathbf{V}\in \mathbb{R}^{n\times r}$ have orthonormal columns and $\boldsymbol{\Sigma} = diag(\sigma_1,\cdots, \sigma_r)$ with $\sigma_1\geq\cdots\geq \sigma_r >0$, where $diag(\mathbf{m})$ denotes the operation that forms a diagonal matrix by putting the elements of $\mathbf{m}$ on the main diagonal.

Given a set of $n$ instances $\{\mathbf{x}_i\}_{i=1}^n \subset \mathbb{R}^d$, let $\mathbf{X} = [\mathbf{x}_1, \cdots, \mathbf{x}_n] \in \mathbb{R}^{d \times n}$ be the data matrix.
If the instances have $V$ different representations or views, then we denote $\mathbf{x}_i = [\mathbf{x}_i^{(1)}; \cdots; \mathbf{x}_i^{(V)}]$, where $\mathbf{x}_i^{(v)} \in \mathbb{R}^{d^{(v)}}$ and $d = \sum \limits_{v=1}^V d^{(v)}$.
Let $\mathbf{X}^{(v)} = [\mathbf{x}_1^{(v)}, \cdots, \mathbf{x}_n^{(v)}] \in \mathbb{R}^{d^{(v)} \times n}$ store the observations on the $v$-th view, then we have $\mathbf{X} = [\mathbf{X}^{(1)}; \cdots; \mathbf{X}^{(V)}] \in \mathbb{R}^{d \times n}$.

\subsection{Related Work}\label{sec_RW}

Most of existing multi-view subspace clustering algorithms are based on the self-representation model, that is, samples from a union of subspaces can be linearly represented by themselves.
Formally, the self-representation model is written as
\begin{equation}\label{self}
    \mathbf{X} = \mathbf{X}\mathbf{Z} + \mathbf{E},
\end{equation}
where $\mathbf{X} \in \mathbb{R}^{d\times n}$ is the data matrix,  $\mathbf{Z} \in \mathbb{R}^{n\times n}$ is the self-representation coefficient matrix, and $\mathbf{E} \in \mathbb{R}^{d\times n}$ is the error matrix.
Then, the general formulation of multi-view subspace clustering methods is
\begin{equation}\label{general}
\begin{array}{cl}
    \min\limits_{\{\mathbf{Z}^{(v)}\}_{v=1}^V} & \sum\limits_{v=1}^V \mathcal{L}(\mathbf{E}^{(v)}) + \lambda \sum\limits_{v=1}^V \Omega(\mathbf{Z}^{(v)}) \\
    \text{s.t.} & \mathbf{X}^{(v)} = \mathbf{X}^{(v)}\mathbf{Z}^{(v)} + \mathbf{E}^{(v)}
\end{array}
\end{equation}
where $\mathbf{X}^{(v)} \in \mathbb{R}^{d^{(v)} \times n}$ and $\mathbf{Z}^{(v)}\in \mathbb{R}^{n \times n}$ are the data matrix and self-representation coefficient matrix of the $v$th view, $\mathcal{L}(\cdot)$ and $\Omega(\cdot)$ are the loss function and regularization term, respectively, and $\lambda > 0$ is the trade-off parameter.
Since the subspace membership is encoded in the learned representation matrix $\{\mathbf{Z}^{(v)}\}_{v=1}^V$, the final segmentation can be obtained by performing spectral clustering algorithms onto the affinity matrix $\mathbf{S} = \sum\limits_{v=1}^V |\mathbf{Z}^{(v)}| + |\mathbf{Z}^{(v)}|^T$.

Various multi-view subspace clustering methods differ in their different choices of loss functions and regularization terms.
The loss function is usually determined by the error types, e.g., $\ell_2$ loss for white noise, $\ell_1$ loss for random corruptions and $\ell_{2,1}$  loss for sample-specific outliers.
Compared with the generality of the loss term, it is the regularization term that characterizes the essential difference between various methods.
For example, to better exploit the complementary information among multiple views, Cao et al. \cite{cao2015DiMSC} employed the Hilbert Schmidt Independence Criterion (HSIC) \cite{gretton2005HSIC} to diversify the self-representation matrices and proposed the Diversity-induced Multi-view Subspace Clustering (DiMSC).
Except for the complementary information, the shared information among views is also of great importance.
Recognizing this, Luo et al. \cite{Luo2018CSMSC} proposed the Consistent and Specific Multi-View Subspace Clustering (CSMSC) by encoding the shared information and the view-unique information into a consistent representation matrix and view-specific representation matrices, respectively.
In order to explore the higher order connection among multi-view representations, the Low-rank Tensor constrained Multiview Subspace Clustering (LTMSC) \cite{zhang2015LTMSC} merges the self-representation matrices into a 3-order tensor and minimizes the tensor rank.
Assuming the samples coming from a union of affine subspaces, Gao et al. \cite{Gao2015MVSC} performed Multi-View Subspace Clustering (MVSC) by constraining the sum of each column of the self-representation matrix to be 1.
Noting that these methods learn self-representation matrices within each view where the data is partially described, Zhang et al. \cite{Zhang2017LMSC} proposed to learn the self-representation matrix of the underlying multi-view latent representation, and the resultant method is called Latent Multi-view Subspace Clustering (LMSC).
However, whether the learned latent representation carries the complete information from multiple views was not analyzed in \cite{Zhang2017LMSC}.
Therefore, whether the data are sufficiently described remains unclear.
Besides, the model of LMSC is actually built on the concatenation of multiple views, possibly leading to the interaction among multiple views being not fully explored and exploited.
Unlike the above methods that only focus on the low-rankness of the representation matrices, several approaches that seek the sparsest and the lowest-rank representations have been proposed \cite{brbic2018MLRSSC,Yin2018SCMV3D}.
For example, Brbi\'{c} and Kopriva \cite{brbic2018MLRSSC} proposed the Multi-view Low-rank Sparse Subspace Clustering (MLRSSC) method via matrix form data, while Yin et al. \cite{Yin2018SCMV3D} aimed to learn a sparse and low-rank representation tenosr for the constructed multi-view tensorial data.

\section{Problem Statement and Analysis}\label{sec_prob_state}

In this paper, we assume that the multiple views of $\{\mathbf{x}_i\}_{i=1}^n$ are generated from a shared latent representation $\mathbf{L}_0 \in \mathbb{R}^{m\times n}$, where $m$ is the dimension of the latent representation.
Formally, this assumption can be described as $\mathbf{X}^{(v)} = \mathbf{G}^{(v)}\mathbf{L}_0 + \mathbf{S}_0^{(v)}$, where $\mathbf{G}^{(v)} \in \mathbb{R}^{d^{(v)}\times m}$ is the transformation matrix, and $\mathbf{S}_0^{(v)} \in \mathbb{R}^{d^{(v)} \times n} $ denotes the errors.
Here, the ``error'' means the deviation between the model assumption and the observed data.
In practice, there are several kinds of errors, such as white noise, missing entries, outliers and corruptions \cite{Liu2018RSP}.
In this paper, the errors is assumed to be gross corruptions \cite{candes2011RPCA}.
That is, the values in $|\mathbf{S}_0^{(v)}|$ ($\forall v$) are elementally sparse and can be arbitrarily large.

Further, in the latent space, the given multi-view samples are assumed to be essentially drawn from a union of multiple subspaces.
That is, $\mathbf{L}_0 $ collects a set of $n$ samples which are exactly drawn from a union of multiple subspaces.
Denote the compact SVD of $\mathbf{L}_0$ as   $\mathbf{U}_0\boldsymbol{\Sigma}_0 \mathbf{V}_0^T$.
As analyzed in \cite{Liu2013LRR,Liu2018RSP,Costeira1998SIM, NIPS2013_SSCLRR}, under certain conditions, the subspace membership of the samples can be determined by the row space of $\mathbf{L}_0$, i.e., $\mathbf{V}_0$, or the Shape Interaction Matrix (SIM) $\mathbf{V}_0\mathbf{V}_0^T$ \cite{Costeira1998SIM} as equal.
Since $\mathbf{L}_0$ is the latent representation of $\{\mathbf{X}^{(v)}\}_{v=1}^V$, the row space of $\mathbf{L}_0$ is called as the latent row space.
With the above hypotheses, it is tacitly assumed that $\mathbf{L}_0$ is the lowest-rank representation that contains the full information from $\{\mathbf{X}_0^{(v)}\}_{v=1}^V$, where $\mathbf{X}_0^{(v)} = \mathbf{G}^{(v)}\mathbf{L}_0 (\forall v)$.

Hence, under the assumptions of this paper, the problem we want to solve can be rewritten mathematically as follows.

\begin{problem}\label{prob1}
Given $n$ samples, which are exactly drawn from a union of multiple subspaces, their multiple observations $\{\mathbf{X}^{(v)}\}_{v=1}^V$ are generated by $\mathbf{X}^{(v)} = \mathbf{G}^{(v)}\mathbf{L}_0 + \mathbf{S}_0^{(v)}$ ($\forall v$), where $\mathbf{L}_0 \in \mathbb{R}^{m\times n}$ is their representation in $\mathbb{R}^m$, and $\mathbf{S}_0^{(v)}$ represents the possible errors and is element-wisely sparse.
Suppose the rank and the compact SVD of $\mathbf{L}_0$ are $r_0$ and $\mathbf{U}_0\boldsymbol{\Sigma}_0 \mathbf{V}_0^T$, respectively.
Given $\{\mathbf{X}^{(v)}\}_{v=1}^V$, the aim is to recover the latent row space identified by $\mathbf{V}_0\mathbf{V}_0^T$, and correct the possible errors $\{\mathbf{S}_0^{(v)}\}_{v=1}^V$.
\end{problem}

Since $\mathbf{L}_0$ is low-rank and $\mathbf{S}_0^{(v)} (\forall v)$ is sparse, to address the above problem, a straightforward solution is
\begin{equation}\label{org}
\begin{array}{cl}
    \min\limits_{\substack{\mathbf{L},\{\mathbf{S}^{(v)}\}_{v=1}^V, \\ \{\mathbf{G}^{(v)}\in \mathbb{R}^{d^{(v)}\times m}\}_{v=1}^V}} & rank(\mathbf{L}) + \lambda \sum\limits_{v=1}^V \|\mathbf{S}^{(v)}\|_0 \\
    \text{s.t.} & \mathbf{X}^{(v)} = \mathbf{G}^{(v)}\mathbf{L} + \mathbf{S}^{(v)}, \forall v
    \end{array}
\end{equation}
where $\lambda > 0$ is a parameter, and $m$ is taken as a parameter to be given in advance. With the estimated $\hat{\mathbf{L}}_0$, the latent row space is obtained by performing SVD on it.

Though the non-convex formulation (\ref{org}) accurately describes Problem \ref{prob1}, the obtained $\hat{\mathbf{L}}_0$ could be arbitrarily close to zero, because scaling $\hat{\mathbf{L}}_0/c$ and $c\hat{\mathbf{G}}^{(v)}$ ($c > 0$ is a constant) has the same loss, where $\hat{\mathbf{G}}^{(v)}$ denotes the solution to ${\mathbf{G}}^{(v)}$ in (\ref{org}).
To avoid the arbitrary scaling of $\mathbf{L}$, one may consider add the constraint that $(\mathbf{G}^{(v)})^T\mathbf{G}^{(v)} = \mathbf{I}$ $(\forall v)$.
Denote the compact SVD of $\mathbf{L}$ as $\mathbf{U}\boldsymbol{\Sigma} \mathbf{V}^T$, then $\mathbf{G}^{(v)}\mathbf{L} = \mathbf{G}^{(v)}\mathbf{U}\boldsymbol{\Sigma} \mathbf{V}^T $.
As $(\mathbf{G}^{(v)}\mathbf{U})^T(\mathbf{G}^{(v)}\mathbf{U}) =\mathbf{I}$ and $\mathbf{V}^T \mathbf{V} = \mathbf{I}$, we know that
$ (\mathbf{G}^{(v)}\mathbf{U})\boldsymbol{\Sigma} \mathbf{V}^T$ is actually an SVD of $\mathbf{G}^{(v)}\mathbf{L}$ ($\forall v$).
Thus, $rank(\mathbf{G}\mathbf{L}) = rank(\mathbf{G}^{(v)}\mathbf{L}) = rank(\mathbf{L})$ $(\forall v)$, and problem (\ref{org}) is equivalent to
\begin{equation}\label{org2}
\begin{array}{cl}
    \min\limits_{\mathbf{L},\mathbf{S}, \mathbf{G}\in \mathbb{R}^{d \times m} } & rank(\mathbf{G}\mathbf{L}) + \lambda  \|\mathbf{S}\|_0 \\
    \text{s.t.} & \mathbf{X} = \mathbf{G}\mathbf{L} + \mathbf{S}\\
    & (\mathbf{G}^{(v)})^T\mathbf{G}^{(v)} = \mathbf{I}, \forall v
    \end{array}
\end{equation}
where $\mathbf{X} = [\mathbf{X}^{(1)}; \cdots; \mathbf{X}^{(V)}]$ is the concatenated feature from multiple views,  $\mathbf{G} = [\mathbf{G}^{(1)}; \cdots; \mathbf{G}^{(V)}]$, and  $\mathbf{S} = [\mathbf{S}^{(1)}; \cdots; \mathbf{S}^{(V)}]$.
Replacing $\mathbf{G}\mathbf{L}$ with $\tilde{\mathbf{L}}$, problem (\ref{org2}) is actually the problem considered by Robust Principle Component Analysis (RPCA) \cite{candes2011RPCA}, i.e.,
\begin{equation}\label{org3}
    \min\limits_{\tilde{\mathbf{L}},\mathbf{S}} ~~ rank(\tilde{\mathbf{L}}) + \lambda  \|\mathbf{S}\|_0 ~~~~
    \text{s.t.} ~~ \mathbf{X} = \tilde{\mathbf{L}} + \mathbf{S}.
\end{equation}

Under some assumptions, the solution of problem (\ref{org3}) is proved to be recovered with high probability by solving \cite{candes2011RPCA}
\begin{equation}\label{PCP}
    \min\limits_{\tilde{\mathbf{L}},\mathbf{S}} ~~ \|\tilde{\mathbf{L}}\|_* + \lambda  \|\mathbf{S}\|_1 ~~~~
    \text{s.t.} ~~ \mathbf{X} = \tilde{\mathbf{L}} + \mathbf{S}.
\end{equation}

In spite of this perfect property, the applicability of (\ref{org3}) is limited, since the condition that $(\mathbf{G}^{(v)})^T\mathbf{G}^{(v)} = \mathbf{I}$ $(\forall v)$ is usually easy to be violated in practice.
Denote $g^{(v)}(\mathbf{L}) = \mathbf{G}^{(v)} \mathbf{L} $.
Actually, according to the ``view insufficiency'' assumption \cite{Xu2015Intact} that each view only carries partial information, the function $g^{(v)}(\mathbf{L})$ is non-invertible.
The non-invertibility of $g^{(v)}(\mathbf{L})$ implies that $\mathbf{G}^{(v)}$ is not column full-rank, then $(\mathbf{G}^{(v)})^T\mathbf{G}^{(v)} = \mathbf{I}$ $(\forall v)$ does not hold.

Hence, to solve Problem \ref{prob1} better, we would like to design a new method, which is called  Latent Complete Row Space Recovery (LCRSR) in next section.

\section{Latent Complete Row Space Recovery for Multi-view Subspace Clustering}\label{sec_MvRSR}
In this section, we present the proposed LCRSR method for more effective and efficient multi-view subspace clustering.

\subsection{Latent Complete  Row Space Recovery}

Denote the rank and the compact SVD of $\mathbf{L}_0$ as $r_0$ and $\mathbf{U}_0\boldsymbol{\Sigma}_0 \mathbf{V}_0^T$, respectively.
Recall that our goal is to obtain $\mathbf{V}_0$ and $\{\mathbf{S}_0^{(v)}\}_{v=1}^V$ with the given $\{\mathbf{X}^{(v)}\}_{v=1}^V$ alone.
To realize this goal, we must establish a relationship between them.
Note that $\mathbf{L}_0 = \mathbf{U}_0\boldsymbol{\Sigma}_0 \mathbf{V}_0^T$ and $\mathbf{V}_0^T(\mathbf{I} - \mathbf{V}_0\mathbf{V}_0^T) = 0$, where the second equality uses the fact that $\mathbf{V}_0^T\mathbf{V}_0 = \mathbf{I}$.
Thus, we can eliminate $\mathbf{L}_0$ via multiplying it by $\mathbf{I} - \mathbf{V}_0\mathbf{V}_0^T$ from the right side \cite{Liu2018RSP}.
That is,
\begin{equation}\label{rowspace}
\mathbf{L}_0(\mathbf{I} - \mathbf{V}_0\mathbf{V}_0^T) = \mathbf{U}_0\boldsymbol{\Sigma}_0 \mathbf{V}_0^T(\mathbf{I} - \mathbf{V}_0\mathbf{V}_0^T) = 0.
\end{equation}
Recall that $\mathbf{X}^{(v)} = \mathbf{G}^{(v)}\mathbf{L}_0 + \mathbf{S}_0^{(v)}$. Then, the relationship between $\mathbf{X}^{(v)}$, $\mathbf{S}_0^{(v)}$ and $\mathbf{V}_0$ is constructed as
\begin{equation}\label{singleRSR}
(\mathbf{X}^{(v)} - \mathbf{S}_0^{(v)})(\mathbf{I} - \mathbf{V}_0\mathbf{V}_0^T) = \mathbf{G}^{(v)}\mathbf{L}_0(\mathbf{I} - \mathbf{V}_0\mathbf{V}_0^T) = 0.
\end{equation}

Denote $\mathbf{X}^{(v)}_0 = \mathbf{X}^{(v)} - \mathbf{S}_0^{(v)} = \mathbf{G}^{(v)}\mathbf{L}_0$ and its compact SVD as $\mathbf{U}_0^{(v)}\boldsymbol{\Sigma}_0^{(v)} (\mathbf{V}_0^{(v)})^T$.
It can be seen that $rank(\mathbf{X}^{(v)}_0) = rank(\mathbf{G}^{(v)}\mathbf{L}_0) \leq rank(\mathbf{L}_0)$.
Moreover, from $\mathbf{X}^{(v)}_0(\mathbf{I} - \mathbf{V}_0\mathbf{V}_0^T) = \mathbf{U}_0^{(v)}\boldsymbol{\Sigma}_0^{(v)} (\mathbf{V}_0^{(v)})^T(\mathbf{I} - \mathbf{V}_0\mathbf{V}_0^T) = 0$, we know that $(\mathbf{V}_0^{(v)})^T(\mathbf{I} - \mathbf{V}_0\mathbf{V}_0^T) = 0$.
That is, the row space of $\mathbf{X}^{(v)}_0$ ($\forall v$) is a subspace of $\mathbf{V}_0$.
In other words, $\mathbf{V}_0$ contains complete information from all views, and is called the latent complete row space.

Based on the above relationship model, the formulation of LCRSR is formulated as
\begin{equation}\label{MvRSP_org}
\begin{array}{cl}
    \min\limits_{\substack{\mathbf{V}\in \mathbb{R}^{n\times r},  \mathbf{V}^T\mathbf{V} = \mathbf{I}, \\ \mathbf{S}^{(v)} \in \mathbb{R}^{d^{(v)} \times n}}} & \sum\limits_{v=1}^V \|\mathbf{S}^{(v)}\|_0 \\
    \text{s.t.} & (\mathbf{X}^{(v)} - \mathbf{S}^{(v)})(\mathbf{I} - \mathbf{V}\mathbf{V}^T) = 0, \forall v
    \end{array}
\end{equation}
where $r$ ($r_0 \le r < \min\{m, n\}$) is a parameter.
When the true rank is known in advance, i.e., $r = r_0$, (\ref{MvRSP_org}) achieves the goal of (\ref{org}) while avoids pushing the obtained $\hat{\mathbf{L}}_0$ to be arbitrarily close to zero.

As a common practice in $\ell_0$ norm minimization problems, we replace the $\ell_0$ norm with the $\ell_1$ norm, leading to
the following formulation:
\begin{equation}\label{MvRSR}
\begin{array}{cl}
\min\limits_{\substack{\mathbf{V}\in \mathbb{R}^{n\times r},  \mathbf{V}^T\mathbf{V} = \mathbf{I}, \\ \mathbf{S}^{(v)} \in \mathbb{R}^{d^{(v)} \times n}}} & \sum \limits_{v=1}^V \|\mathbf{S}^{(v)}\|_1 \\
\text{s.t.} & (\mathbf{X}^{(v)} - \mathbf{S}^{(v)})(\mathbf{I} - \mathbf{V}\mathbf{V}^T) = 0, \forall v.\\
\end{array}
\end{equation}

It is worth noting that LCRSR aims to learn the latent complete row space by fixing the rank explicitly, i.e., $\mathbf{V}\in \mathbb{R}^{n\times r}$ and $r$ is fixed a \textit{priori}.
The reasonability of this fixed-rank parameterization can be explained as follows.
First, in some applications, the rank of the low-rank part is known in advance.
For instance, in motion segmentation, the feature trajectories of each video can be approximately modeled as samples drawn from a union of linear subspaces of dimension at most 4 \cite{motion4d}.
For another instance, under the Lambertian assumption, the face images of subjects approximately lie in a union of 9-dimensional linear subspaces \cite{face9d}.
Second, in methods (e.g., CSMSC\cite{Luo2018CSMSC}, and LMSC \cite{Zhang2017LMSC}) that minimize the rank of the representation matrix in the objective, the varying of the trade-off parameter will change the rank of the solution.
For example, increasing the $\lambda$ in (\ref{org2}) results in lower $\|\mathbf{S}\|_{0}$ but higher rank of $\mathbf{GL}$.
Using a fixed-rank in LCRSR avoids this kind of trade-off.

\subsection{Subspace Recovery Property of LCRSR}
We first consider the ``ideal'' case that there are no gross errors in the representations of multiple views, i.e., $\mathbf{X}^{(v)} = \mathbf{X}_0^{(v)} = \mathbf{G}^{(v)}\mathbf{L}_0$ ($\forall v$).
In this case, problems (\ref{MvRSP_org}) and (\ref{MvRSR}) are equivalent and will reduce to solving the following equation with $\mathbf{P}$ being the unknown variable:
\begin{equation}\label{clean1}
    \mathbf{X}^{(v)}(\mathbf{I} - \mathbf{P}) = 0, \forall v,  ~~~\text{s.t.}~ \mathbf{P} \in \mathcal{P},
\end{equation}
where $\mathcal{P} = \{\mathbf{V}\mathbf{V}^T| \mathbf{V}\in \mathbb{R}^{n\times r}, \mathbf{V}^T\mathbf{V} = \mathbf{I}\}$.
Suppose the true rank is known in advance, i.e., $r = r_0$, we show that $\mathbf{P} =\mathbf{V}_0 \mathbf{V}_0^T$ is the only feasible solution to problem (\ref{clean1}).

\begin{prop}\label{TH_clean}
If $r = r_0$ is known in advance, then the row space of $\mathbf{L}_0$ (identified by $\mathbf{V}_0\mathbf{V}_0^T$) is exactly recovered by solving (\ref{clean1}) when $\mathbf{S}_0^{(v)} = 0$ ($\forall v$).
\end{prop}

\begin{proof}
Since $\mathbf{X}^{(v)}  = \mathbf{G}^{(v)}\mathbf{L}_0 = \mathbf{G}^{(v)}\mathbf{U}_0\boldsymbol{\Sigma}_0 \mathbf{V}_0^T$ and $\mathbf{V}_0^T(\mathbf{I} - \mathbf{V}_0\mathbf{V}_0^T) = 0$, it holds that $\mathbf{X}^{(v)}(\mathbf{I} - \mathbf{V}_0\mathbf{V}_0^T) = 0$.
That is, $\mathbf{V}_0\mathbf{V}_0^T$ is a solution to (\ref{clean1}).

Suppose there is another $\tilde{\mathbf{V}} \in \mathbb{R}^{n\times r_0}$ satisfying $\tilde{\mathbf{V}}^T\tilde{\mathbf{V}} = \mathbf{I}$, and $\tilde{\mathbf{V}}\tilde{\mathbf{V}}^T$ is a solution to (\ref{clean1}).
Then, we have $\mathbf{X}^{(v)}(\mathbf{I} - \tilde{\mathbf{V}}\tilde{\mathbf{V}}^T ) = \mathbf{U}^{(v)}\boldsymbol{\Sigma}^{(v)} (\mathbf{V}^{(v)})^T(\mathbf{I} - \tilde{\mathbf{V}}\tilde{\mathbf{V}}^T ) =  0$, where $\mathbf{U}^{(v)}\boldsymbol{\Sigma}^{(v)} (\mathbf{V}^{(v)})^T$ is the compact SVD of $\mathbf{X}^{(v)}$.
Since both $\mathbf{V}_0$ and $ \tilde{\mathbf{V}}$ are the solutions to (\ref{clean1}), it holds that $\bigcup \limits_{v=1}^V \mathbf{V}^{(v)}$  $=(\bigcup \limits_{v=1}^V \mathbf{V}^{(v)})\mathbf{V}_0\mathbf{V}_0^T$  $=(\bigcup \limits_{v=1}^V \mathbf{V}^{(v)})\tilde{\mathbf{V}}\tilde{\mathbf{V}}^T$.
Now we prove $\mathbf{V}_0\mathbf{V}_0^T = \tilde{\mathbf{V}}\tilde{\mathbf{V}}^T$ by contradiction.
If $\mathbf{V}_0\mathbf{V}_0^T \neq \tilde{\mathbf{V}}\tilde{\mathbf{V}}^T$, then there exists $\bar{\mathbf{V}}$, which is a proper subspace of both $\mathbf{V}_0$ and $\tilde{\mathbf{V}}$, such that $\bigcup \limits_{v=1}^V \mathbf{V}^{(v)} = (\bigcup \limits_{v=1}^V \mathbf{V}^{(v)})\bar{\mathbf{V}}\bar{\mathbf{V}}^T$.
Note that $rank(\bar{\mathbf{V}}) < r_0$, which contradicts the fact that $\mathbf{L}_0$ is the lowest-rank representation that contains the full information from $\{\mathbf{X}_0^{(v)}\}_{v=1}^V$.
Thus, $\mathbf{V}_0\mathbf{V}_0^T = \tilde{\mathbf{V}}\tilde{\mathbf{V}}^T$.
That is, $\mathbf{V}_0\mathbf{V}_0^T $ is the only feasible solution to (\ref{clean1}) when $\mathbf{S}_0^{(v)}= 0$ ($\forall v$).
\end{proof}

When there are gross errors, i.e., there exists $ \mathbf{S}_0^{(v')} ( v' \in \{1, 2, \cdots, V\})$, such that $\mathbf{S}_0^{(v')} \neq 0$, we consider the following equivalent formulation:
\begin{equation}\label{errorcase}
\begin{split}
\min\limits_{\{\mathbf{S}^{(v)}\}_{v=1}^V, \mathbf{P} \in \mathcal{P}} & \sum\limits_{v=1}^V \|\mathbf{S}^{(v)}\|_1, \\
\text{s.t.} ~~~~~~&     (\mathbf{X}^{(v)}-\mathbf{S}^{(v)})(\mathbf{I} - \mathbf{P}) = 0, \forall v.
\end{split}
\end{equation}
Given $r = r_0$, based on the analysis of the ``ideal'' case, it is known that $\mathbf{P} = \mathbf{V}_0\mathbf{V}_0^T$ is the only feasible solution to problem (\ref{errorcase}) if $\mathbf{S}^{(v)} = \mathbf{S}_0^{(v)} (\forall v)$.
On the other hand, if $\mathbf{P} =\mathbf{V}_0 \mathbf{V}_0^T$, then problem (\ref{errorcase}) can be decomposed into $V$ sparse signal recovery \cite{Candes2005RIP} problems.
For the $v$th view, the problem is
\begin{equation}\label{ssr}
\min\limits_{\mathbf{s}^{(v)}} \|\mathbf{s}^{(v)}\|_1, ~
\text{s.t.} ~ vec((\mathbf{X}^{(v)}(\mathbf{I} - \mathbf{P})) = ((\mathbf{I} - \mathbf{P})\otimes \mathbf{I}) \mathbf{s}^{(v)},
\end{equation}
where $\otimes$ is the Kronecker product, and $vec(\cdot)$ is the operator vectorizing a matrix into a vector along the vertical direction.
If $(\mathbf{I} - \mathbf{P})\otimes \mathbf{I} $ satisfies the Restricted Isometry Property, then $vec(\mathbf{S}_0^{(v)})$ may be recovered by the convex program in (\ref{ssr})  \cite{Liu2018RSP,Candes2005RIP}.
Though it is not easy to obtain an exact conclusion, the above analysis shows that $(\mathbf{V}_0\mathbf{V}_0^T, \{\mathbf{S}_0^{(v)}\}_{v=1}^V)$ is very likely to be a critical point to problem (\ref{errorcase}).

\subsection{Clustering Procedure}
Once the estimated latent complete row space $\hat{\mathbf{V}}_0$ is obtained, we need to calculate the final clustering results based on it.
For this sake, we first present the following analysis.

If sufficient samples are drawn from each of the independent subspaces, then the authentic latent complete row space $\mathbf{V}_0$ will satisfy that $\mathbf{V}_0\mathbf{V}_0^T$ is block-diagonal \cite{Liu2013LRR,Liu2018RSP,Costeira1998SIM,NIPS2013_SSCLRR}.
In fact, under this assumption, $\mathbf{V}_0$ itself also possesses a block-diagonal structure \cite{Liu2018RSP}.
Without loss of generality, suppose the samples from the same subspace are arranged together.
That is, $\mathbf{L}_0 = [\mathbf{L}_1, \cdots, \mathbf{L}_k]$, where $\mathbf{L}_i$ ($1\le i \le k$) denotes the matrix of samples from the $i$-th subspace, and $k$ is the number of subspaces.
Let $\mathbf{U}_i\boldsymbol{\Sigma}_i\mathbf{V}_i^T$ be the compact SVD of $\mathbf{L}_i$. Recall that $\mathbf{U}_0\boldsymbol{\Sigma}_0 \mathbf{V}_0^T$ is the compact SVD of $\mathbf{L}_0$, then it holds that
\[\mathbf{V}_0 = \bar{\mathbf{V}}_0\mathbf{Q}, \bar{\mathbf{V}}_0 = \begin{bmatrix}
\mathbf{V}_1 &  & \\
& \ddots & \\
& & \mathbf{V}_k
\end{bmatrix},\]
where $\mathbf{Q} \in \mathbb{R}^{r\times r}$ is an orthogonal matrix \cite{Liu2018RSP}.
This implies that $\mathbf{V}_0$ is equivalent to a block-diagonal matrix.
On one hand, due to the block diagonal property of $\bar{\mathbf{V}}_0$, one could achieve right clustering by implementing the K-Means algorithm directly to partition the rows of $\bar{\mathbf{V}}_0$ into $k$ groups.
On the other hand, the inner products among row vectors of $\bar{\mathbf{V}}_0$ will not be changed when multiplying it by the orthogonal matrix $\mathbf{Q}$ on the right.
Therefore, one will obtain the same partition when the inputs to K-Means is replaced by $\mathbf{V}_0$.

As revealed by the above analysis, performing K-Means on the row vectors of $\hat{\mathbf{V}}_0$ is a reasonable way to obtain the final clustering results.
Algorithm \ref{alg:2} displays the whole schemes of the proposed multi-view subspace clustering method.

\begin{algorithm}
\caption{Multi-view subspace clustering by LCRSR}
\label{alg:2}
\begin{algorithmic}
\STATE \textbf{Input:}
$\{\mathbf{X}^{(v)}\}_{v=1}^V$,  number of clusters $k$.
\STATE \textbf{Parameters:} $r > 0$, $\lambda > 0$.
\STATE \textbf{Output:}
clustering results.
\STATE 1: Input $\{\mathbf{X}^{(v)}\}_{v=1}^V$, $r$ and $\lambda$,  and solve (\ref{MvRSR}) to obtain $\hat{\mathbf{V}}_0$.
\STATE 2: Implement K-Mean onto the row vectors of $\hat{\mathbf{V}}_0$ to obtain $k$ clusters.
\end{algorithmic}
\end{algorithm}

\section{Optimization of LCRSR}\label{sec_opt}
We consider to approximate the solution to problem (\ref{MvRSR}) by solving
\begin{equation}\label{MvRSR4}
\begin{array}{cl}
\min\limits_{\substack{\mathbf{V}\in \mathbb{R}^{n\times r},  \mathbf{V}^T\mathbf{V} = \mathbf{I}, \\ \mathbf{S}^{(v)} \in \mathbb{R}^{d^{(v)} \times n}}} & \lambda \sum \limits_{v=1}^V \|\mathbf{S}^{(v)}\|_1 \\
& +\sum\limits_{v=1}^V\|(\mathbf{X}^{(v)} - \mathbf{S}^{(v)})(\mathbf{I} - \mathbf{V}\mathbf{V}^T)\|.\\
\end{array}
\end{equation}
where $\lambda > 0$ is a trade-off parameter, and the non-squared Frobenius norm is used to increase the robustness against large error terms.
For more meticulous modeling, one can equip $\mathbf{S}^{(v)}$ with a view-specific parameter $\lambda^{(v)} > 0$.
Here, to avoid tedious parameter tuning and over-fitting, we just use a unified trade-off parameter.

\subsection{Optimization Procedure}
The utilization of non-squared Frobenius norm directly in the loss term makes it difficult to address problem (\ref{MvRSR4}).
Fortunately, as we will shown in Sec. \ref{sec_alo_analysis}, problem (\ref{MvRSR4}) can be addressed by solving the following problem
\begin{equation}\label{MvRSP5}
\begin{array}{cl}
\min\limits_{\substack{\mathbf{V}\in \mathbb{R}^{n\times r},  \mathbf{V}^T\mathbf{V} = \mathbf{I}, \\ \mathbf{S}^{(v)} \in \mathbb{R}^{d^{(v)} \times n}}} & \lambda \sum \limits_{v=1}^V \|\mathbf{S}^{(v)}\|_1 \\
& + \sum\limits_{v=1}^Vp^{(v)}\|(\mathbf{X}^{(v)} - \mathbf{S}^{(v)})(\mathbf{I} - \mathbf{V}\mathbf{V}^T)\|^2,\\
\end{array}
\end{equation}
where
\begin{equation}\label{def_p}
    p^{(v)} = \frac{1}{2\|(\mathbf{X}^{(v)} - \mathbf{S}^{(v)})(\mathbf{I} - \mathbf{V}\mathbf{V}^T)\|}
\end{equation}
can be seen as an automatically determined weight for the $v$th view.
To avoid being divided by 0, we can let  $p^{(v)} = \frac{1}{2\|(\mathbf{X}^{(v)} - \mathbf{S}^{(v)})(\mathbf{I} - \mathbf{V}\mathbf{V}^T)\| + \epsilon}$, where $\epsilon > 0$ is a very small constant.

Denote $\mathbf{p} = [p^{(1)}, \cdots, p^{(V)}]^T$.
Since $p^{(v)}$ is dependent on $\mathbf{S}^{(v)}$ and $\mathbf{V}$, $\mathbf{p}$ is also unknown.
Hence, there are in total three groups of unknown variables in problem (\ref{MvRSP5}).
The alternating minimization strategy is employed to optimize them.
Suppose $(\mathbf{V}_t, \{\mathbf{S}_t^{(v)}\}_{v=1}^V, \mathbf{p}_t)$ is the solution obtained at the $t$th iteration.
Denote
\begin{equation}\label{gv}
    g^{(v)}(\mathbf{V}, \mathbf{S}^{(v)}, p^{(v)}) = p^{(v)}\|(\mathbf{X}^{(v)} - \mathbf{S}^{(v)})(\mathbf{I} - \mathbf{V}\mathbf{V}^T)\|^2.
\end{equation}
Then, problem (\ref{MvRSP5}) is solved by iterating the following three procedures.

\subsubsection{Updating $\mathbf{V}$ by fixing the others.}
The $\mathbf{V}$-subproblem is
\begin{equation}\label{v_sub}
    \min\limits_{\mathbf{V}^T\mathbf{V} = \mathbf{I}} \sum\limits_{v=1}^V g^{(v)}(\mathbf{V}, \mathbf{S}_t^{(v)}, p_t^{(v)}).
\end{equation}
It is equivalent to
\begin{equation}\label{v_sub2}
    \max\limits_{\mathbf{V}^T\mathbf{V} = \mathbf{I}} Tr(\mathbf{V}^T\sum\limits_{v=1}^V {p_t^{(v)}}(\mathbf{X}^{(v)} - \mathbf{S}_t^{(v)})^T(\mathbf{X}^{(v)} - \mathbf{S}_t^{(v)})\mathbf{V}),
\end{equation}
which has a closed-form solution.
That is, $\mathbf{V}_{t+1}$ is formed by the top $r$ eigenvectors (corresponding to the $r$ maximum eigenvalues) of a semi-positive definite matrix $\mathbf{M}_t = \sum\limits_{v=1}^V {p_t^{(v)}}(\mathbf{X}^{(v)} - \mathbf{S}_t^{(v)})^T(\mathbf{X}^{(v)} - \mathbf{S}_t^{(v)})$.

\subsubsection{Updating $\mathbf{S}^{(v)}$ by fixing the others.}
Since $\{\mathbf{S}^{(v)}\}_{v=1}^V$ are independent from each other, we can solve each $\mathbf{S}^{(v)}$ individually.
The $\mathbf{S}^{(v)}$-subproblem is addressed by the proximal gradient-descent method \cite{Attouch2009ConProx, Parikh2014proximal}:
\begin{equation}\label{s_sub}
\begin{array}{l}
    \mathbf{S}_{t+1}^{(v)} = \arg\min\limits_{\mathbf{S}^{(v)}} \frac{\lambda}{\mu_t^{(v)}}\|\mathbf{S}^{(v)}\|_1 \\
     \quad ~+ \frac{1}{2}\|\mathbf{S}^{(v)} - (\mathbf{S}^{(v)}_t - \frac{\left.\partial_{\mathbf{S}^{(v)}} g^{(v)}(\mathbf{V}_{t+1}, \mathbf{S}^{(v)}, p_{t}^{(v)})\right|_{ \mathbf{S}_t^{(v)}}}{\mu_t^{(v)}})\|^2,
    \end{array}
\end{equation}
where $\mu_t^{(v)} > 0$ is a penalty parameter, and
\begin{equation}\label{s_sub2}
    \partial_{\mathbf{S}^{(v)}} g^{(v)}(\mathbf{V}_{t+1}, \mathbf{S}^{(v)}, p_t^{(v)}) \!=\! {2p_t^{(v)}}(\mathbf{S}^{(v)}\!-\!\mathbf{X}^{(v)})(\mathbf{I}\!-\!\mathbf{V}_{t+1}\mathbf{V}_{t+1}^T),
\end{equation}
is the partial derivative of $g^{(v)}(\mathbf{V}, \mathbf{S}^{(v)}, p^{(v)})$ w.r.t. $\mathbf{S}^{(v)}$ at $\mathbf{V} = \mathbf{V}_{t+1}$ and $p^{(v)} = p_t^{(v)}$.
According to \cite{Parikh2014proximal}, the penalty parameter could be set as $\mu_t^{(v)} = {2p_t^{(v)}}{\|\mathbf{I}-\mathbf{V}_{t+1}\mathbf{V}_{t+1}^T\|_2} = 2p_t^{(v)}$\footnote{Since $\mathbf{V}^T\mathbf{V} = \mathbf{I}$, it can be inferred that $\|\mathbf{I}-\mathbf{V}_{t+1}\mathbf{V}_{t+1}^T\|_2 = 1$.}, or determined by backtracking line search.
Therefore, the solution to the $\mathbf{S}^{(v)}$-subproblem is given by
\begin{equation}\label{s_sub4}
\begin{array}{ll}
    \mathbf{S}_{t+1}^{(v)} \!&= \!\mathcal{H}_{\frac{\lambda}{\mu_t^{(v)}}}\! \left[\mathbf{S}^{(v)}_t \!- \! \frac{\left.\partial_{\mathbf{S}^{(v)}} g^{(v)}(\mathbf{V}_{t+1}, \mathbf{S}^{(v)}, p_t^{(v)})\right|_{ \mathbf{S}_t^{(v)}}}{\mu_t^{(v)}}\!\right],\\
    &= \mathcal{H}_{\frac{\lambda}{\mu_t^{(v)}}}\! \left[\mathbf{S}^{(v)}_t \!- \! \frac{{2p_t^{(v)}} {(\mathbf{S}_t^{(v)}-\mathbf{X}^{(v)})(\mathbf{I}-\mathbf{V}_{t+1}\mathbf{V}_{t+1}^T)}
}{\mu_t^{(v)}}\!\right].
\end{array}
\end{equation}
$\mathcal{H}_{\lambda/\mu_t^{(v)}} $ denotes the element-wise shrinkage operator with parameter $\lambda/\mu_t^{(v)}$, and it is defined as
\begin{equation}
    \mathcal{H}_{\varepsilon}[\mathbf{x}] = \text{sign}(\mathbf{x})\max(0, |\mathbf{x}|-\varepsilon).
\end{equation}

\subsubsection{Updating $\mathbf{p}$ by fixing the others.}
According to Eq. (\ref{def_p}), $p_{t+1}^{(v)}$ is updated by
\begin{equation}\label{updateP}
    p_{t+1}^{(v)}= \frac{1}{2\|(\mathbf{X}^{(v)} - \mathbf{S}_{t+1}^{(v)})(\mathbf{I} - \mathbf{V}_{t+1}\mathbf{V}_{t+1}^T)\|}, \forall v.
\end{equation}

Algorithm \ref{alg:LCRSR} summarizes the whole optimization procedure.

\begin{algorithm}
\caption{The algorithm to solve the LCRSR model}
\label{alg:LCRSR}
\begin{algorithmic}
\STATE \textbf{Input:}
$\{\mathbf{X}^{(v)}\}_{v=1}^V$,  $r > 0$, $\lambda > 0$.
\STATE \textbf{Initialization:}
$\mathbf{S}_t^{(v)} = \mathbf{0}$, $p_t^{(v)} = 1, \forall v$, $t = 0$.
\STATE \textbf{Output:}
$\{\mathbf{S}^{(v)}\}_{v=1}^V$ and $\mathbf{V}$.
\STATE \textbf{repeat} not converged \textbf{do}
\STATE 1: \quad Compute $\mathbf{M}_t = \sum\limits_{v=1}^V {p_t^{(v)}}(\mathbf{X}^{(v)} - \mathbf{S}_t^{(v)})^T(\mathbf{X}^{(v)} - \mathbf{S}_t^{(v)})$.
\STATE 2: \quad Update $\mathbf{V}_{t+1}$ using the top $r$ eigenvectors of $\mathbf{M}_t$.
\STATE 3: \quad Update $\mathbf{S}_{t+1}^{(v)}$ according to Eq. (\ref{s_sub4}), $\forall v$.
\STATE 4: \quad Update ${p}_{t+1}^{(v)} (\forall v)$ by Eq. (\ref{updateP}).

\STATE 6: \quad $t = t+1$.
\STATE \textbf{until} meeting the stopping criterion
\end{algorithmic}
\end{algorithm}

\subsection{Convergence Analysis}\label{sec_alo_analysis}
In this subsection, we prove that Algorithm \ref{alg:LCRSR} will monotonically decrease the objective value of problem (\ref{MvRSR4}) until the convergence.

\begin{prop}\label{TH1}
Algorithm \ref{alg:LCRSR} will monotonically decrease the objective value of Eq. (\ref{MvRSR4}) until the convergence.
\end{prop}

\begin{proof}
Suppose after the $t$th iteration, we have obtained $\mathbf{V}_t, \{\mathbf{S}_t^{(v)}\}$, and $\mathbf{p}_t$ with $p_t^{(v)}= \frac{1}{2\|(\mathbf{X}^{(v)} - \mathbf{S}_{t}^{(v)})(\mathbf{I} - \mathbf{V}_{t}\mathbf{V}_{t}^T)\|}$, $\forall v$.
In the $(t+1)$th iteration, with $\{\mathbf{S}^{(v)}\}$ and $\mathbf{p}$ being $\{\mathbf{S}_t^{(v)}\}$ and $\mathbf{p}_t$ respectively, we address subproblem (\ref{v_sub}) to obtain $\mathbf{V}_{t+1}$.
Then, we have
\begin{equation}\label{proof1}
   \sum\limits_{v=1}^V g^{(v)}(\mathbf{V}_{t+1}, \mathbf{S}_t^{(v)}, p_t^{(v)}) \leq \sum\limits_{v=1}^V g^{(v)}(\mathbf{V}_t, \mathbf{S}_t^{(v)}, p_t^{(v)}).
\end{equation}
Denote the whole objective function in Eq. (\ref{MvRSP5}) as $h(\mathbf{V}, \{\mathbf{S}^{(v)}\}, \mathbf{p}) = \sum\limits_{v=1}^V \lambda\|\mathbf{S}^{(v)}\|_1 + g^{(v)}(\mathbf{V}, \mathbf{S}^{(v)}, p^{(v)})$. Inequality (\ref{proof1}) implies
\begin{equation}\label{proof2}
    h(\mathbf{V}_{t+1}, \{\mathbf{S}_t^{(v)}\}, \mathbf{p}_t) \leq h(\mathbf{V}_{t}, \{\mathbf{S}_t^{(v)}\}, \mathbf{p}_t).
\end{equation}
Then, we update $\{\mathbf{S}^{(v)}\}$ by solving subproblem (\ref{s_sub}) view-by-view.
Due to the convergence property of the proximal gradient-decent method \cite{Beck2009FISTA}, it holds that
\begin{equation}\label{proof3}
    h(\mathbf{V}_{t+1}, \{\mathbf{S}_{t+1}^{(v)}\}, \mathbf{p}_t) \leq h(\mathbf{V}_{t+1}, \{\mathbf{S}_t^{(v)}\}, \mathbf{p}_t).
\end{equation}
With inequalities (\ref{proof2}) and (\ref{proof3}), we have $h(\mathbf{V}_{t+1}, \{\mathbf{S}_{t+1}^{(v)}\}, \mathbf{p}_t)$ $\leq$ $h(\mathbf{V}_{t}, \{\mathbf{S}_t^{(v)}\}, \mathbf{p}_t)$.
Unfolding this inequality, it is
\begin{equation}\label{proof4}
\begin{split}
\sum\limits_{v=1}^V \lambda \|\mathbf{S}_{t+1}^{(v)}\|_1 + p_t^{(v)}\|(\mathbf{X}^{(v)} - \mathbf{S}_{t+1}^{(v)})(\mathbf{I} - \mathbf{V}_{t+1}\mathbf{V}_{t+1}^T)\|^2 \\
    \leq \sum\limits_{v=1}^V \lambda \|\mathbf{S}_{t}^{(v)}\|_1 + p_t^{(v)}\|(\mathbf{X}^{(v)} - \mathbf{S}_t^{(v)})(\mathbf{I} - \mathbf{V}_{t}\mathbf{V}_{t}^T)\|^2.
    \end{split}
\end{equation}
Substituting $p_t^{(v)} = \frac{1}{2\|(\mathbf{X}^{(v)} - \mathbf{S}_{t}^{(v)})(\mathbf{I} - \mathbf{V}_{t}\mathbf{V}_{t}^T)\|}$ into the above inequality, we have
\begin{equation}\label{proof5}
\begin{split}
\sum\limits_{v=1}^V \lambda \|\mathbf{S}_{t+1}^{(v)}\|_1 + \sum\limits_{v=1}^V\frac{\|(\mathbf{X}^{(v)} - \mathbf{S}_{t+1}^{(v)})(\mathbf{I} - \mathbf{V}_{t+1}\mathbf{V}_{t+1}^T)\|^2}{2\|(\mathbf{X}^{(v)} - \mathbf{S}_{t}^{(v)})(\mathbf{I} - \mathbf{V}_{t}\mathbf{V}_{t}^T)\|} \\
    \leq  \sum\limits_{v=1}^V \lambda \|\mathbf{S}_{t}^{(v)}\|_1 + \sum\limits_{v=1}^V\frac{\|(\mathbf{X}^{(v)} - \mathbf{S}_t^{(v)})(\mathbf{I} - \mathbf{V}_{t}\mathbf{V}_{t}^T)\|^2}{2\|(\mathbf{X}^{(v)} - \mathbf{S}_{t}^{(v)})(\mathbf{I} - \mathbf{V}_{t}\mathbf{V}_{t}^T)\|}.
    \end{split}
\end{equation}

Note that $2ab \leq a^2 + b^2 \Rightarrow b - \frac{b^2}{2a} \leq a - \frac{a^2}{2a} (a \geq 0)$, we have
\begin{equation}\label{proof6}
\begin{array}{l}
\|(\mathbf{X}^{(v)} \!- \! \mathbf{S}_{t+1}^{(v)})(\mathbf{I} \!-\! \mathbf{V}_{t+1}\mathbf{V}_{t+1}^T)\| \!- \! \frac{\|(\mathbf{X}^{(v)} \!-\! \mathbf{S}_{t+1}^{(v)})(\mathbf{I}\! -\! \mathbf{V}_{t+1}\mathbf{V}_{t+1}^T)\|^2}{2\|(\mathbf{X}^{(v)} \!- \!\mathbf{S}_{t}^{(v)})(\mathbf{I}\! - \! \mathbf{V}_{t}\mathbf{V}_{t}^T)\|} \\
    \leq \|(\mathbf{X}^{(v)} \!-\! \mathbf{S}_t^{(v)})(\mathbf{I} \!-\! \mathbf{V}_{t}\mathbf{V}_{t}^T)\| \!- \! \frac{\|(\mathbf{X}^{(v)} \!-\! \mathbf{S}_t^{(v)})(\mathbf{I} \!-\! \mathbf{V}_{t}\mathbf{V}_{t}^T)\|^2}{2\|(\mathbf{X}^{(v)} \!-\! \mathbf{S}_{t}^{(v)})(\mathbf{I} \!-\! \mathbf{V}_{t}\mathbf{V}_{t}^T)\|}.
    \end{array}
\end{equation}
Summing the inequality in (\ref{proof5}) over all views and combining the inequality (\ref{proof4}), we arrive at
\begin{equation}\label{proof7}
\begin{array}{l}
\sum\limits_{v=1}^V \lambda \|\mathbf{S}_{t+1}^{(v)}\|_1 + \sum\limits_{v=1}^V\|(\mathbf{X}^{(v)} \!- \! \mathbf{S}_{t+1}^{(v)})(\mathbf{I} \!-\! \mathbf{V}_{t+1}\mathbf{V}_{t+1}^T)\|  \\
    \leq \sum\limits_{v=1}^V \lambda \|\mathbf{S}_{t}^{(v)}\|_1 + \sum\limits_{v=1}^V\|(\mathbf{X}^{(v)} \!-\! \mathbf{S}_t^{(v)})(\mathbf{I} \!-\! \mathbf{V}_{t}\mathbf{V}_{t}^T)\|.
    \end{array}
\end{equation}
That is, the objective value of Eq. (\ref{MvRSR4}) is monotonically decreased.
Recall that the objective function is lower bounded by 0.
Therefore, the objective value sequence will converge.
\end{proof}

\subsection{Computational Complexity}
Algorithm \ref{alg:2} has mainly two steps. In Step 1, the LCRSR algorithm is conducted to obtain $\hat{\mathbf{V}}_0$. Then, Step 2 performs K-Means to obtain the final clustering results, which consumes only $\mathcal{O}(nkr)$.

As for Algorithm \ref{alg:LCRSR}, in Step 1, calculating $\mathbf{M}_t = \sum\limits_{v=1}^V {p_t^{(v)}}(\mathbf{X}^{(v)} - \mathbf{S}_t^{(v)})^T(\mathbf{X}^{(v)} - \mathbf{S}_t^{(v)})$ costs $2nd + n^2d$ elementary additions or multiplications, where $d = \sum\limits_{v=1}^V d^{(v)}$.
Then, in Step 2, $\mathbf{V}_{t+1}$ is updated by calculating the top $r$ eigenvectors of the $n\times n$ matrix $\mathbf{M}_t$.
This can be accomplished by performing the partial eigenvalue decomposition, which takes $\mathcal{O}(n^2r)$ \cite{Golub2013MC}.
To update $\{\mathbf{S}_{t+1}^{(v)}\}_{v=1}^V$, we need first to calculate the gradient defined in Eq. (\ref{s_sub2}), which consumes $\mathcal{O}(n^2(d+r))$.
Then the shrinkage operation spends $\mathcal{O}(nd)$.
In Step 5, when updating $\mathbf{p}_{t+1}$, since the calculations of $\mathbf{X}^{(v)} - \mathbf{S}_{t+1}^{(v)}$ and $\mathbf{I} - \mathbf{V}_{t+1}\mathbf{V}_{t+1}^T$ can be shared in the procedures of calculating $\mathbf{M}_{t+1}$ and updating $\{\mathbf{S}_{t+1}^{(v)}\}_{v=1}^V$,
this step takes extra $\mathcal{O}(n^2d)$ for multiplying $\mathbf{X}^{(v)} - \mathbf{S}_{t+1}^{(v)}$ by $\mathbf{I} - \mathbf{V}_{t+1}\mathbf{V}_{t+1}^T$ on all views and $\mathcal{O}(nd)$ for calculating the Frobenius norms.
In total, Algorithm \ref{alg:LCRSR} has a time complexity of $O(n^2(d+r))$ for each iteration.

Denote $T$ as the number of iterations spent by Algorithm \ref{alg:LCRSR} to converge, the overall complexity of Algorithm \ref{alg:2} is $O(n^2(d+r)T + nkr)$.

\section{Experiment}\label{sec_exp}
In this section, experiments are conducted to examine the performance of the proposed LCRSR approach.
We first test LCRSR's ability to recover the latent complete row space on synthetic data.
Then, we investigate its clustering performance and apply it to background subtraction from multi-view videos. Finally, experiments on the sensitivity to parameters, convergence behavior and computational time are conducted.

\subsection{Data Descriptions}
We consider performing experiments on 8 datasets, including 2 UCI datasets, 1 text-gene dataset, 1 image-text dataset, 2 image dataset and 2 video datasets.
Concretely, the Dermatology dataset\footnote{https://archive.ics.uci.edu/ml/datasets/dermatology}  (Derm) and the Forest type mapping dataset\footnote{https://archive.ics.uci.edu/ml/datasets/Forest+type+mapping} (Forest) are downloaded from the UCI machine learning repository.
The text-gene and image-text datasets are the prokaryotic phyla dataset\footnote{https://github.com/mbrbic/MultiViewLRSSC/tree/master/datasets} (Prok) and Wikipedia articles\footnote{http://www.svcl.ucsd.edu/projects/crossmodal/} (Wiki), respectively.
The image datasets includes a 7-class subset of the Caltech101\footnote{http://www.vision.caltech.edu/Image Datasets/Caltech101/} (Caltech7) and the USPS digits\footnote{http://www.csie.ntu.edu.tw/$\sim$cjlin/libsvmtools/datasets/multiclass.html\#usps }.
The used video datasets are the EPFL Laboratory  sequences\footnote{https://cvlab.epfl.ch/data/data-pom-index-php/} (Lab), and Dongzhimen Transport Hub Crowd\footnote{http://www.escience.cn/people/huyongli/Dongzhimen.html} (DTHC).

\textbf{1) Derm}: This dataset was collected to differentiate the type of Eryhemato-Squamous diseases in dermatology.
There are 366 patients diagnosed with 6 diseases.
For each patient, there are 11 clinical features (the age feature is discarded due to missing values) and 22 histopathological features.
The two kinds of features correspond to the clinical view and histopathological view of this dataset, respectively.

\textbf{2) Forest}: This is a multi-temporal remote sensing dataset.
It aims to tell apart different forest types using the spectral data derived from the ASTER satellite imagery.
It contains 524 instances belonging to 4 classes.
There are 2 kinds of feature representations. The first kind includes 9 reflected spectral features in the green, red and near-infrared bands in the ASTER images. The second kind is composed of 18 geographically weighted similarity variables \cite{johnson2012ForestData}.

\textbf{3) Prok}: It consists of 551 prokaryotic samples belonging to 4 classes. The species are represented by 1 textual view and 2 genomic views \cite{brbic2016landscape}.
The textual descriptions are summarized into a document-term matrix that records the TF-IDF \cite{Salton1988TFIDF} re-weighted word frequencies.
The genomic views are the proteome composition and the gene repertoire.

\textbf{4) Wiki}: This database is composed of 2,866 image-text documents classified into 10 categories. The documents were collected from the Wikipedia's featured articles. The articles were segmented into sections according to section headings, and the images were grouped according to their sections where they were positioned by the author(s).

\textbf{5) Caltech7}: This subset includes 441 images of 7 object classes: Dolla-Bill, Faces, Garfield, Motorbike, Snoopy, Stop-Sign and Windsor-Chair. To form a multi-view dataset, three kinds of features are extracted, they are, SIFT \cite{lowe2004SIFT}, GIST \cite{oliva2001GIST} and HOG \cite{dalal2005HOG}.

\textbf{6) USPS}: This database consists of 9,298 images of 10 handwritten digits (0 to 9). Each image is of size $16\times 16$.
To meet the multi-view setting, we extract the SIFT feature and GIST feature as two views.

\textbf{7) Lab}: This is a multi-camera pedestrian video shot inside a laboratory by 4 cameras \cite{2011LabData}.
Each sequence is composed of 3915 frames, recording the event that four people entered the room in sequence and walked around.
The frames are converted into gray images and the resolution of each image is reduced from 288 $\times$ 360 to 144 $\times$ 180.
Finally, the size of the data matrix for each view is $29250\times 3915$.
We display some example frames in Fig. \ref{Fig.ExpImga}.

\textbf{8) DTHC}: The video clips in this dataset are captured from real scenarios at the Dongzhiimen Transport Hub in Beijing. There are 3 cameras deployed in a hall to record the actions of passengers \cite{2017WangLapLRR}.
We use the video clips from the category named ``Dispersing from the center quickly'' for experiments.
For each of the three sequences, there are 151 synchronized frames with resolution $1080 \times 1920$.
Similarly, we convert the frames to gray images and the resolution is reduced to $135 \times 240$.
The resultant data matrix is with size $32400\times 151$ for each view.
Some frames from this dataset are shown in Fig. \ref{Fig.ExpImgb}.

\begin{figure}
\centering
\subfigure[]{
\includegraphics[width=0.45\textwidth]{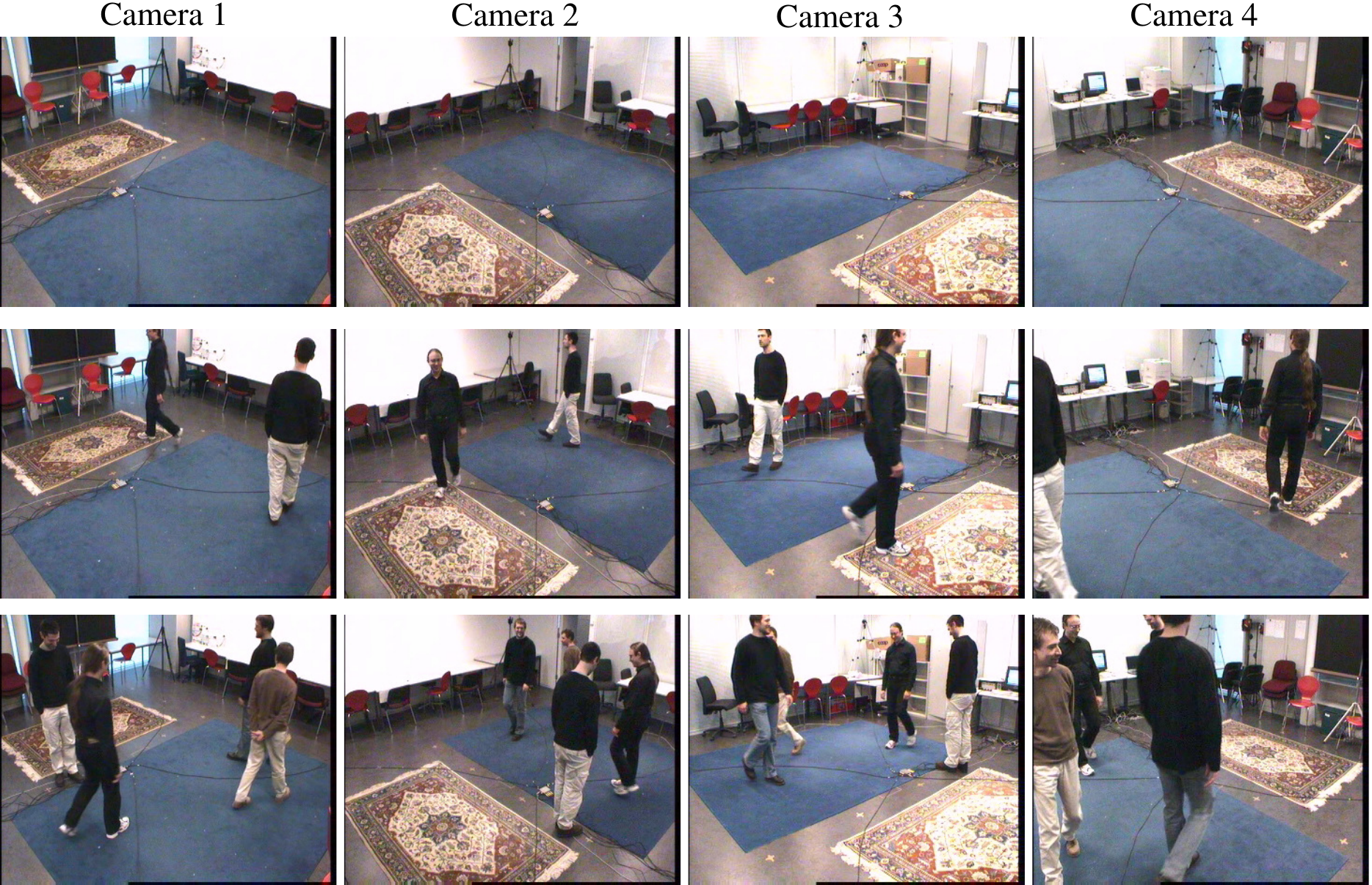}\label{Fig.ExpImga}}
\subfigure[]{
\includegraphics[width=0.45\textwidth]{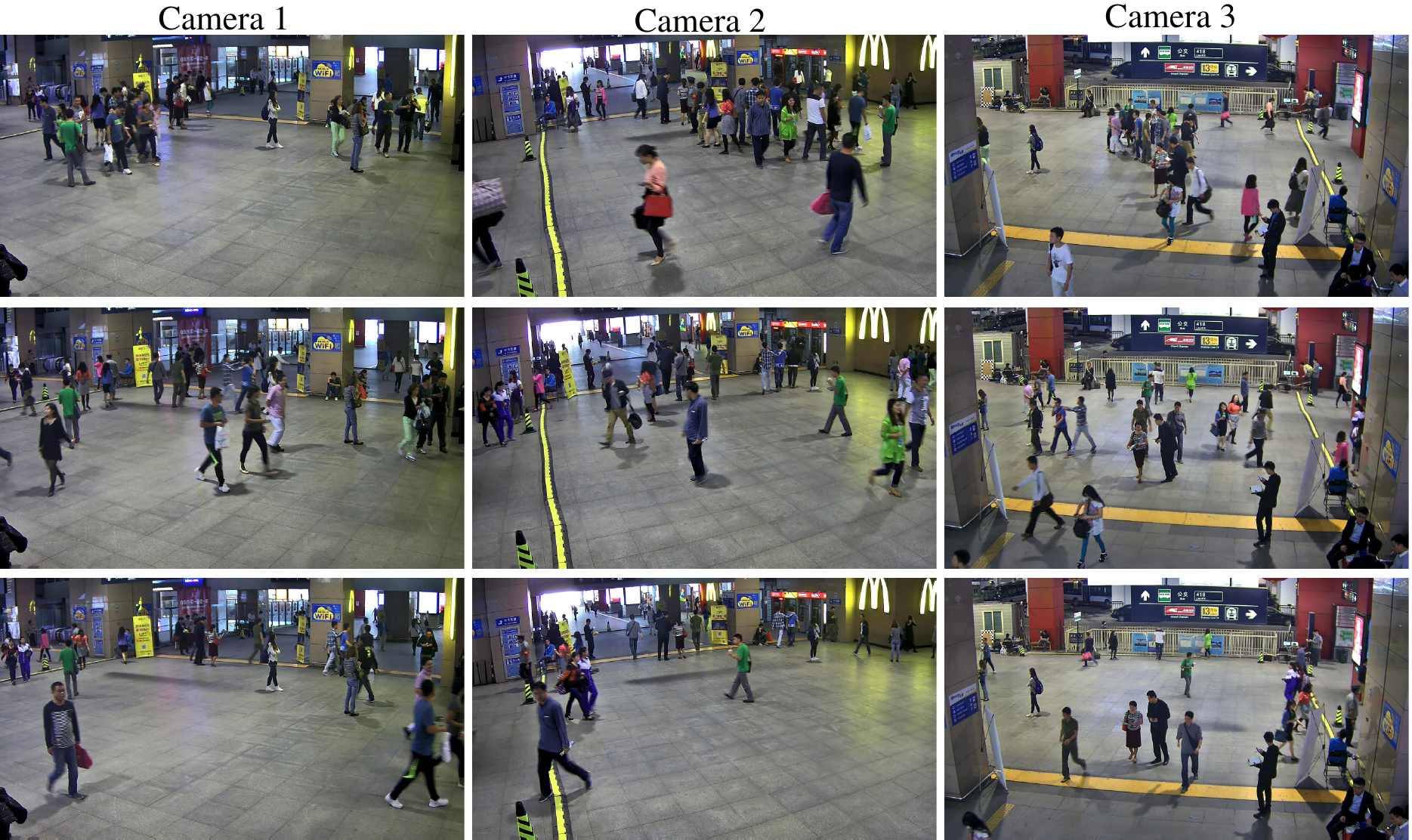}\label{Fig.ExpImgb}}
\centering
\caption{Example frames from the video datasets. Images in the same row are shot at the same time. (a) Lab. (b) DTHC. }
\label{Fig.ExpImg}
\end{figure}

\begin{figure}
  \includegraphics[width=0.45\textwidth]{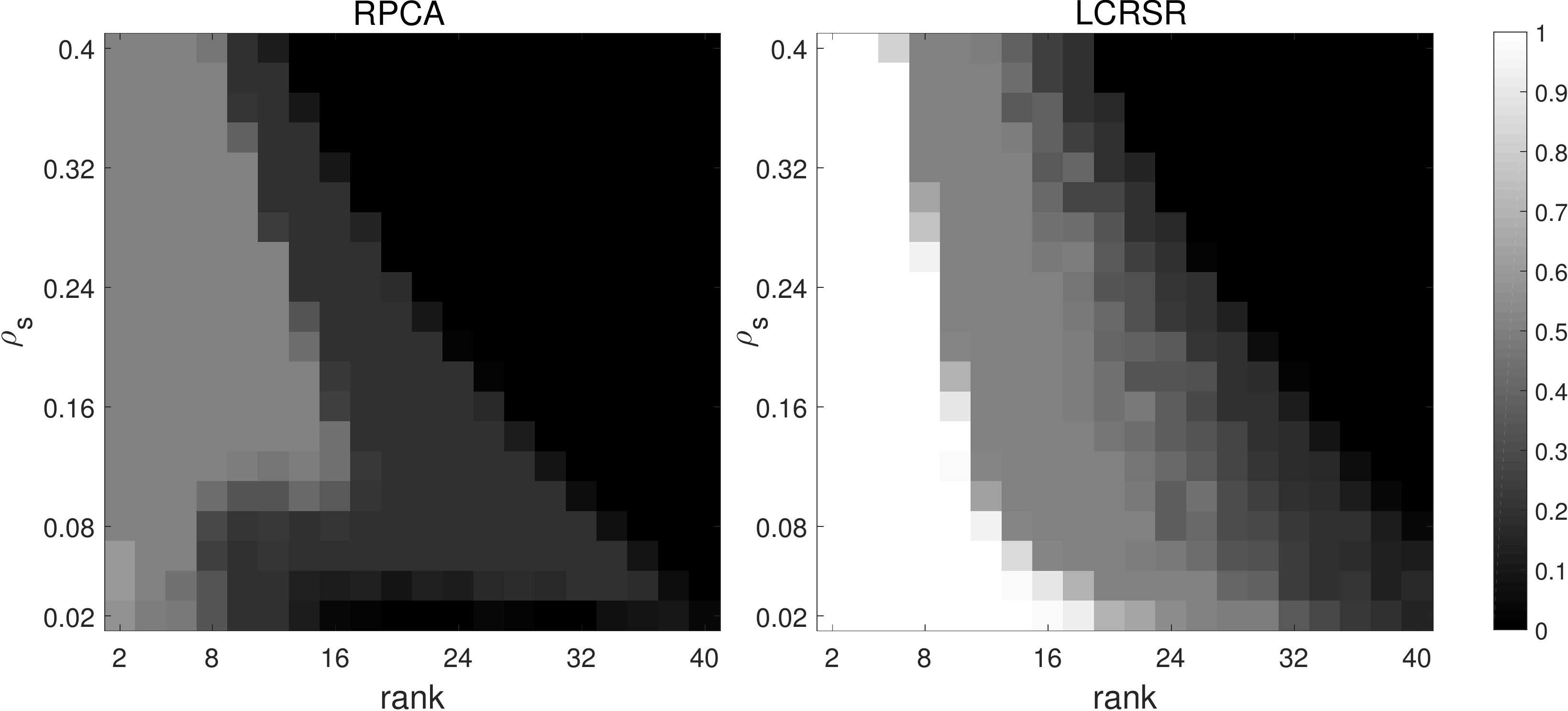}\\
  \caption{Latent row space recovery results on the synthetic data. }\label{figVrecover}
\end{figure}

\begin{table*}
  \centering
  \caption{Clustering performance comparison (mean $\pm $ std). Symbols `$\star/-/\triangledown$' indicate the winning/drawing/losing of LCRSR when compared with the corresponding method under the Wilcoxon signed-rank test with confidence level 0.01. The highest scores are in bold.}\label{Table1}
 \begin{tabular}{cccccc}
    \hline
Datasets	& Methods		&NMI	&ACC	&F	&AdjRI\\ \hline
\multirow{11}*{Derm}	&BestLRR	&0.766 $\pm$ 0.000$\star$	 &0.781 $\pm$ 0.000$\star$	&0.720 $\pm$ 0.000$\star$	&0.654 $\pm$ 0.000$\star$\\
	&ConLRR	&0.849 $\pm$ 0.000$\star$	&0.932 $\pm$ 0.000$\star$	 &0.878 $\pm$ 0.000$\star$	&0.848 $\pm$ 0.000$\star$\\
	&BestRPCA	&0.699 $\pm$ 0.003$\star$	&0.631 $\pm$ 0.019$\star$	 &0.663 $\pm$ 0.005$\star$	&0.586 $\pm$ 0.008$\star$\\
	&ConRPCA	&0.675 $\pm$ 0.000$\star$	&0.724 $\pm$ 0.000$\star$	 &0.680 $\pm$ 0.000$\star$	&0.608 $\pm$ 0.000$\star$\\
	&LMSC	&0.823 $\pm$ 0.003$\star$	&0.885 $\pm$ 0.001$\star$	 &0.859 $\pm$ 0.002$\star$	&0.825 $\pm$ 0.002$\star$\\
	&MVSC	&0.710 $\pm$ 0.000$\star$	&0.798 $\pm$ 0.000$\star$	 &0.758 $\pm$ 0.000$\star$	&0.690 $\pm$ 0.000$\star$\\
	&CSMSC	&0.902 $\pm$ 0.000$\star$	&0.959 $\pm$ 0.000$\star$	 &0.933 $\pm$ 0.000$\star$	&0.917 $\pm$ 0.000$\star$\\
	&LTMSC	&0.818 $\pm$ 0.003$\star$	&0.913 $\pm$ 0.001$\star$	 &0.877 $\pm$ 0.002$\star$	&0.847 $\pm$ 0.002$\star$\\
	&MLRSSC	&0.820 $\pm$ 0.002$\star$	&0.866 $\pm$ 0.001$\star$	 &0.841 $\pm$ 0.001$\star$	&0.803 $\pm$ 0.001$\star$\\
	&LCRSR	&\textbf{0.922} $\pm$ 0.006	&\textbf{0.963} $\pm$ 0.003	 &\textbf{0.948} $\pm$ 0.005	&\textbf{0.935} $\pm$ 0.006\\ \hline
\multirow{11}*{Forest}	&BestLRR	&0.513 $\pm$ 0.000$\star$	 &0.782 $\pm$ 0.000$\star$	&0.632 $\pm$ 0.001$\star$	&0.491 $\pm$ 0.001$\star$\\
	&ConLRR	&0.563 $\pm$ 0.003$\star$	&0.818 $\pm$ 0.001$\star$	 &0.682 $\pm$ 0.002$\star$	&0.561 $\pm$ 0.003$\star$\\
	&BestRPCA	&0.478 $\pm$ 0.000$\star$	&0.711 $\pm$ 0.000$\star$	 &0.583 $\pm$ 0.000$\star$	&0.430 $\pm$ 0.000$\star$\\
	&ConRPCA	&0.563 $\pm$ 0.000$\star$	&0.807 $\pm$ 0.000$\star$	 &0.662 $\pm$ 0.000$\star$	&0.533 $\pm$ 0.000$\star$\\
	&LMSC	&0.552 $\pm$ 0.001$\star$	&0.803 $\pm$ 0.000$\star$	 &0.660 $\pm$ 0.000$\star$	&0.529 $\pm$ 0.001$\star$\\
	&MVSC	&0.374 $\pm$ 0.002$\star$	&0.494 $\pm$ 0.001$\star$	 &0.518 $\pm$ 0.000$\star$	&0.193 $\pm$ 0.000$\star$\\
	&CSMSC	&0.576 $\pm$ 0.001$\star$	&0.819 $\pm$ 0.001$\star$	 &0.682 $\pm$ 0.001$\star$	&0.560 $\pm$ 0.002$\star$\\
	&LTMSC	&0.478 $\pm$ 0.001$\star$	&0.713 $\pm$ 0.001$\star$	 &0.564 $\pm$ 0.000$\star$	&0.396 $\pm$ 0.001$\star$\\
	&MLRSSC	&0.490 $\pm$ 0.000$\star$	&0.730 $\pm$ 0.000$\star$	 &0.596 $\pm$ 0.000$\star$	&0.435 $\pm$ 0.000$\star$\\
	&LCRSR	&\textbf{0.609} $\pm$ 0.010	&\textbf{0.844} $\pm$ 0.006	 &\textbf{0.735} $\pm$ 0.008	&\textbf{0.627 }$\pm$ 0.011\\ \hline
\multirow{11}*{Prok}	&BestLRR	&0.162 $\pm$ 0.003$\star$	 &0.410 $\pm$ 0.003$\star$	&0.362 $\pm$ 0.001$\star$	&0.063 $\pm$ 0.002$\star$\\
	&ConLRR	&0.152 $\pm$ 0.003$\star$	&0.546 $\pm$ 0.001$\star$	 &0.539 $\pm$ 0.000$\star$	&0.054 $\pm$ 0.001$\star$\\
	&BestRPCA	&0.208 $\pm$ 0.010$\star$	&0.563 $\pm$ 0.009$\star$	 &0.493 $\pm$ 0.007$\star$	&0.184 $\pm$ 0.008$\star$\\
	&ConRPCA	&0.126 $\pm$ 0.015$\star$	&0.521 $\pm$ 0.007$\star$	 &0.510 $\pm$ 0.002$\star$	&0.103 $\pm$ 0.004$\star$\\
	&LMSC	&0.384 $\pm$ 0.000$\star$	&0.635 $\pm$ 0.000$\star$	 &0.553 $\pm$ 0.000$\star$	&0.339 $\pm$ 0.000$\star$\\
	&MVSC	&0.416 $\pm$ 0.000$\star$	&0.653 $\pm$ 0.000$-$	 &\textbf{0.607} $\pm$ 0.000$-$	&0.318 $\pm$ 0.000$\star$\\
	&CSMSC	&0.305 $\pm$ 0.005$\star$	&0.625 $\pm$ 0.022$\star$	 &0.535 $\pm$ 0.026$\star$	&0.288 $\pm$ 0.036$\star$\\
	&LTMSC	&0.121 $\pm$ 0.005$\star$	&0.407 $\pm$ 0.016$\star$	 &0.388 $\pm$ 0.009$\star$	&0.022 $\pm$ 0.002$\star$\\
	&MLRSSC	&0.382 $\pm$ 0.025$\star$	&0.653 $\pm$ 0.030$-$	&0.526 $\pm$ 0.039$\star$	&0.307 $\pm$ 0.053$\star$\\
	&LCRSR	&\textbf{0.445} $\pm$ 0.065	&\textbf{0.676} $\pm$ 0.075	 &0.595 $\pm$ 0.070	&\textbf{0.390} $\pm$ 0.098\\ \hline
\multirow{11}*{Wiki}	&BestLRR	&0.523 $\pm$ 0.000$\star$	 &0.538 $\pm$ 0.000$\star$	&0.479 $\pm$ 0.000$\star$	&0.417 $\pm$ 0.000$\star$\\
	&ConLRR	&0.440 $\pm$ 0.001$\star$	&0.479 $\pm$ 0.001$\star$	 &0.418 $\pm$ 0.001$\star$	&0.348 $\pm$ 0.001$\star$\\
	&BestRPCA	&0.209 $\pm$ 0.000$\star$	&0.291 $\pm$ 0.000$\star$	 &0.208 $\pm$ 0.000$\star$	&0.116 $\pm$ 0.000$\star$\\
	&ConRPCA	&0.061 $\pm$ 0.000$\star$	&0.193 $\pm$ 0.000$\star$	 &0.130 $\pm$ 0.000$\star$	&0.029 $\pm$ 0.000$\star$\\
	&LMSC	&0.508 $\pm$ 0.005$\star$	&0.575 $\pm$ 0.005$\star$	 &0.505 $\pm$ 0.004$\star$	&0.446 $\pm$ 0.004$\star$\\
	&MVSC	&0.455 $\pm$ 0.002$\star$	&0.544 $\pm$ 0.002$\star$	 &0.453 $\pm$ 0.002$\star$	&0.383 $\pm$ 0.003$\star$\\
	&CSMSC	&0.225 $\pm$ 0.002$\star$	&0.331 $\pm$ 0.002$\star$	 &0.245 $\pm$ 0.001$\star$	&0.154 $\pm$ 0.001$\star$\\
	&LTMSC	&0.443 $\pm$ 0.002$\star$	&0.515 $\pm$ 0.002$\star$	 &0.450 $\pm$ 0.003$\star$	&0.384 $\pm$ 0.003$\star$\\
	&MLRSSC	&0.535 $\pm$ 0.003$\star$	&0.587 $\pm$ 0.006$\star$	 &0.522 $\pm$ 0.007$\star$	&0.465 $\pm$ 0.007$\star$\\
	&LCRSR	&\textbf{0.563} $\pm$ 0.001	&\textbf{0.609} $\pm$ 0.002	 &\textbf{0.530} $\pm$ 0.003	&\textbf{0.475} $\pm$ 0.003\\ \hline
\multirow{11}*{Caltech7}	&BestLRR	&0.632 $\pm$ 0.000$\star$	 &0.712 $\pm$ 0.000$\star$	&0.635 $\pm$ 0.000$\star$	&0.566 $\pm$ 0.000$\star$\\
	&ConLRR	&0.717 $\pm$ 0.000$\star$	&0.746 $\pm$ 0.000$\star$	 &0.708 $\pm$ 0.000$\star$	&0.650 $\pm$ 0.000$\star$\\
	&BestRPCA	&0.647 $\pm$ 0.001$\star$	&0.721 $\pm$ 0.001$\star$	 &0.645 $\pm$ 0.001$\star$	&0.577 $\pm$ 0.002$\star$\\
	&ConRPCA	&0.728 $\pm$ 0.001$\star$	&0.755 $\pm$ 0.001$\star$	 &0.717 $\pm$ 0.001$\star$	&0.662 $\pm$ 0.001$\star$\\
	&LMSC	&0.724 $\pm$ 0.004$\star$	&0.748 $\pm$ 0.003$\star$	 &0.718 $\pm$ 0.004$\star$	&0.666 $\pm$ 0.005$\star$\\
	&MVSC	&0.666 $\pm$ 0.000$\star$	&0.773 $\pm$ 0.000$\star$	 &0.655 $\pm$ 0.000$\star$	&0.579 $\pm$ 0.000$\star$\\
	&CSMSC	&0.674 $\pm$ 0.000$\star$	&0.741 $\pm$ 0.000$\star$	 &0.617 $\pm$ 0.000$\star$	&0.532 $\pm$ 0.000$\star$\\
	&LTMSC	&0.709 $\pm$ 0.001$\star$	&0.726 $\pm$ 0.000$\star$	 &0.659 $\pm$ 0.000$\star$	&0.588 $\pm$ 0.000$\star$\\
	&MLRSSC	&0.714 $\pm$ 0.005$\star$	&0.736 $\pm$ 0.002$\star$	 &0.700 $\pm$ 0.003$\star$	&0.644 $\pm$ 0.003$\star$\\
	&LCRSR	&\textbf{0.775} $\pm$ 0.023	&\textbf{0.797} $\pm$ 0.032	 &\textbf{0.773} $\pm$ 0.030	&\textbf{0.728} $\pm$ 0.037\\ \hline
\multirow{11}*{USPS}	&BestLRR	&0.747 $\pm$ 0.000$\star$	 &0.628 $\pm$ 0.000$\star$	&0.640 $\pm$ 0.000$\star$	&0.596 $\pm$ 0.000$\star$\\
	&ConLRR	&0.747 $\pm$ 0.000$\star$	&0.619 $\pm$ 0.001$\star$	 &0.640 $\pm$ 0.000$\star$	&0.596 $\pm$ 0.000$\star$\\
	&BestRPCA	&0.736 $\pm$ 0.000$\star$	&0.645 $\pm$ 0.000$\star$	 &0.634 $\pm$ 0.001$\star$	&0.589 $\pm$ 0.001$\star$\\
	&ConRPCA	&0.746 $\pm$ 0.003$\star$	&0.618 $\pm$ 0.006$\star$	 &0.639 $\pm$ 0.006$\star$	&0.595 $\pm$ 0.008$\star$\\
	&LMSC	&0.652 $\pm$ 0.007$\star$	&0.662 $\pm$ 0.001$\star$	 &0.612 $\pm$ 0.002$\star$	&0.564 $\pm$ 0.003$\star$\\
	&MVSC	&0.607 $\pm$ 0.001$\star$	&0.493 $\pm$ 0.008$\star$	 &0.485 $\pm$ 0.004$\star$	&0.416 $\pm$ 0.005$\star$\\
	&CSMSC	&0.757 $\pm$ 0.000$\star$	&0.635 $\pm$ 0.000$\star$	 &0.658 $\pm$ 0.000$\star$	&0.616 $\pm$ 0.000$\star$\\
	&LTMSC	&0.746 $\pm$ 0.004$\star$	&0.649 $\pm$ 0.015$\star$	 &0.649 $\pm$ 0.011$\star$	&0.607 $\pm$ 0.013$\star$\\
	&MLRSSC	&0.710 $\pm$ 0.001$\star$	&0.702 $\pm$ 0.001$\star$	 &0.661 $\pm$ 0.001$\star$	&0.621 $\pm$ 0.001$\star$\\
	&LCRSR	&\textbf{0.769} $\pm$ 0.001	&\textbf{0.764} $\pm$ 0.023	 &\textbf{0.734} $\pm$ 0.012	&\textbf{0.702} $\pm$ 0.014\\ \hline

  \end{tabular}
\end{table*}
\subsection{Baselines}
We would like to make comparisons in terms of clustering performance with the following competitors.
\begin{itemize}
  \item We first consider to make comparison with LRR \cite{Liu2013LRR}. Respectively, the best single-view LRR results and the results of LRR on the concatenation of multiple views form two baselines, termed ``BestLRR'' and ``ConLRR''.

  \item Second, the results of the classical Robust Principle Component Analysis (RPCA) \cite{candes2011RPCA} are compared. On each view or the concatenated data matrix, RPCA is first implemented to estimate the low-rank representation. Then LRR is applied on the low-rank representation to obtain the final clustering results.
      This leads to two competitors, called ``BestRPCA'' and ``ConRPCA''.

  \item Besides the above single-view subspace clustering methods, we also compare with the following multi-view subspace clustering methods. They are, LMSC \cite{Zhang2017LMSC}, MVSC \cite{Gao2015MVSC}, CSMSC \cite{Luo2018CSMSC}, LTMSC \cite{zhang2015LTMSC} and MLRSSC \cite{brbic2018MLRSSC}.
\end{itemize}

Except for the proposed LCRSR, all the baselines methods need to perform spectral clustering onto the learned affinity matrices to obtain the final clustering results. We consider directly run spectral clustering on the learned affinity matrix without any additional complex post-processing on it.

\begin{figure*}
\centering
\subfigure[]{
\includegraphics[width=0.48\textwidth]{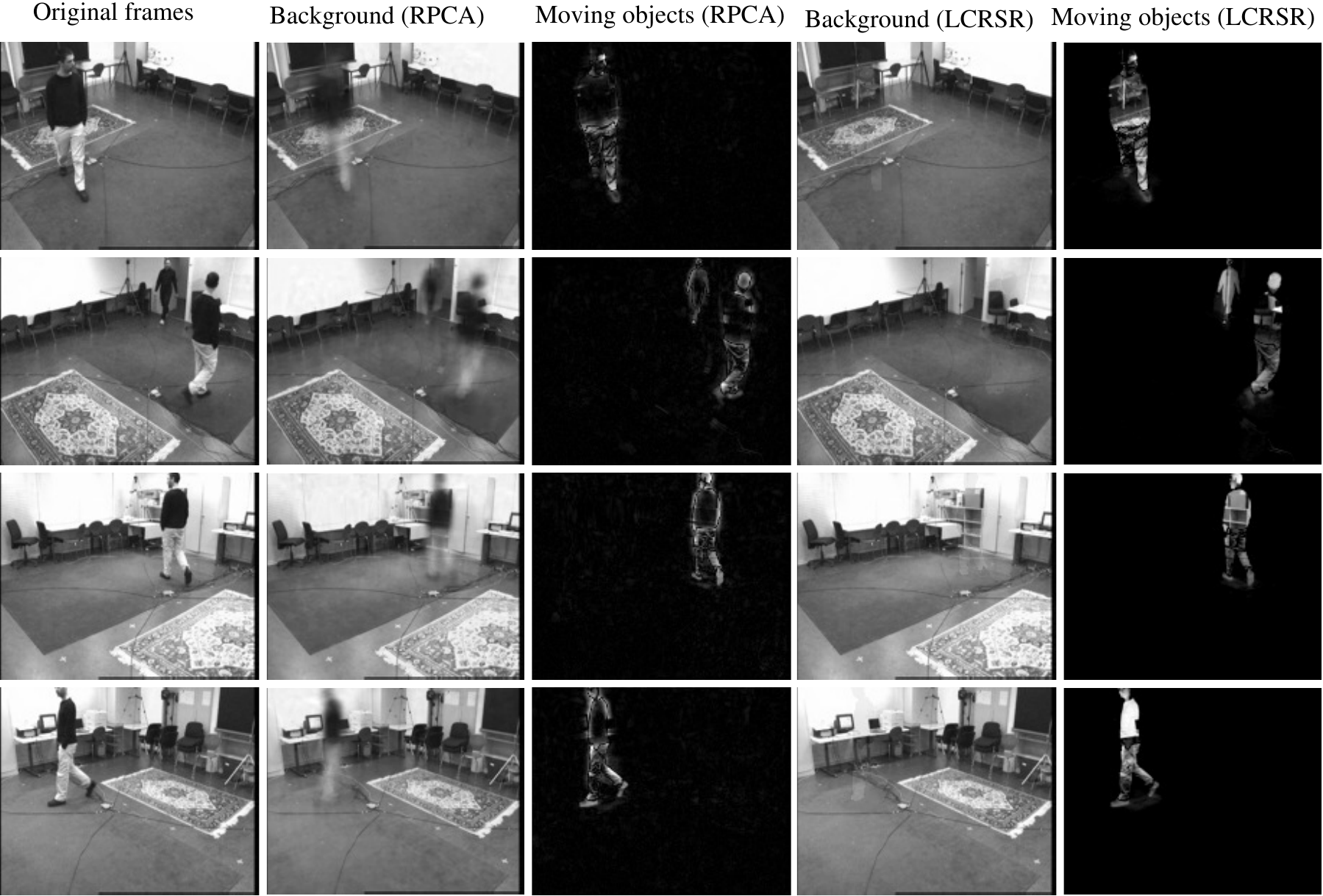}\label{Fig.Moving1a}}
\subfigure[]{
\includegraphics[width=0.48\textwidth]{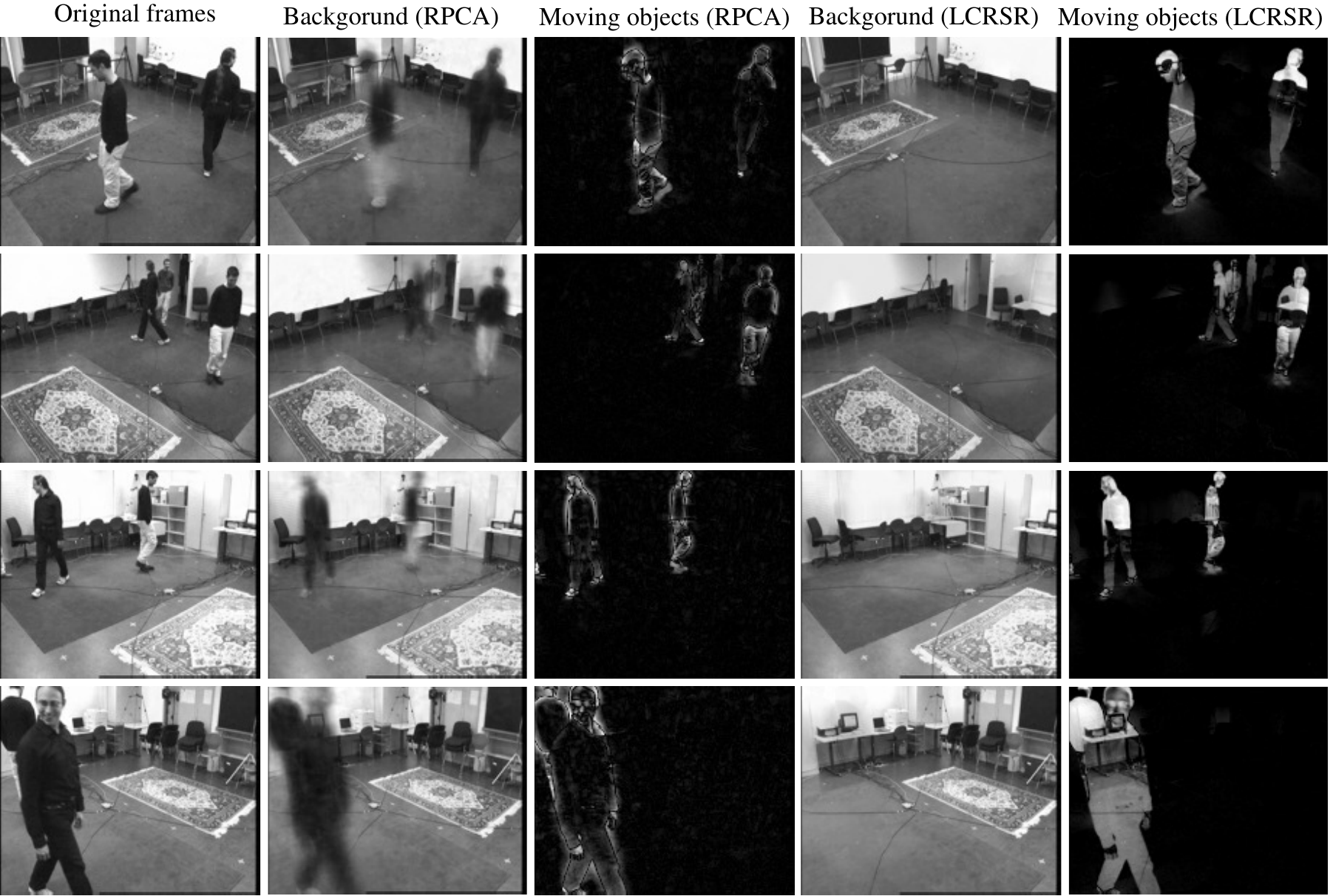}\label{Fig.Moving1b}}
\centering
\caption{Background subtraction in multi-camera surveillance video. Frames at two time points from the Lab video dataset. Each row corresponds to a camera. The parameters in LCRSR are set as $r = 5$ and $\lambda = 10$.}
\label{Fig.Moving1}
\end{figure*}

\begin{figure*}
\centering
\includegraphics[width=0.98\textwidth]{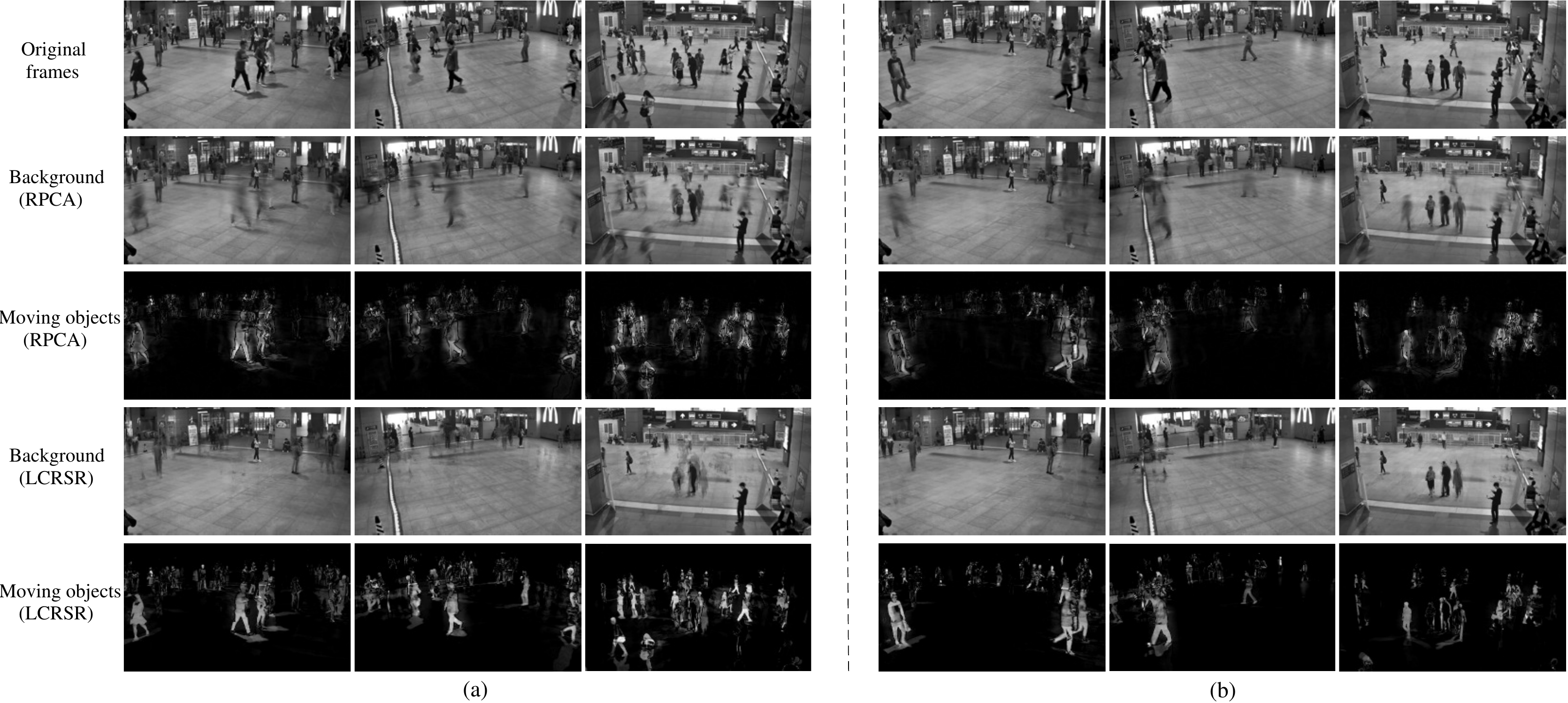}
\centering
\caption{Background subtraction in multi-camera surveillance video. Frames at two time points from the DTHC video dataset. Each column corresponds to a camera. The parameters in LCRSR are set as $r = 3$ and $\lambda = 10$.}
\label{Fig.Moving2}
\end{figure*}

\subsection{Experimental Settings}
For the proposed LCRSR, there are two parameters: $r$ and $\lambda$.
$r$ is manually tuned to maximize the performance of LCRSR, and $\lambda$ is selected within the candidate set $\{10^{-6}, 10^{-5}, \cdots, 10^{3}\}$.
For LRR, its parameter $\lambda$ is tuned within $\{0.01, 0.05, 0.1, 0.5, 1, \cdots 4\}$.
As suggested in \cite{candes2011RPCA}, the parameter in RPCA is set as $\lambda = 1/\sqrt{\max(n, \tilde{d})}$, where $\tilde{d}$ equals to $d$ (ConRPCA) or $d^{(v)}$ (single-view RPCA).
Regarding LMSC, except for the regularization parameter $\lambda$, it also needs to set the dimension of the latent representation.
As for the latent dimension, it is tuned from $\{5, 10, \cdots, 30\}$ for Derm and Forest, which are with low dimensionality.
For the rest datasets with relatively higher dimensionality, the latent dimension is set as 100, as suggested by the authors \cite{Zhang2017LMSC}.
Then, $\lambda$ in LMSC is tuned from $\{10^{-3}, 10^{-2}, \cdots, 10^{3}\}$.
The two parameters $\lambda_1$ and $\lambda_2$ of MVSC are both chosen from $\{10^{-3}, 10^{-2}, 10^{-1}, 1\}$.
As for CSMSC, both $\lambda_C$ and $\lambda_D$ are selected from $\{10^{-3}, 10^{-2}, 10^{-1}, 1\}$.
The regularization parameter $\lambda$ in LTMSC is chosen from $\{10^{-2}, 10^{-1}, 10^{0}, 10^{1}, 10^{2}\}$.
For MLRSSC, we compared with its centroid-based version \cite{brbic2018MLRSSC}. As suggested in \cite{brbic2018MLRSSC}, the low-rank parameter $\beta_1$ is tuned from 0.1 to 0.9 with step 0.2, and the sparsity parameter is set as $1-\beta_1$, and the consensus parameter is tuned from 0.3 to 0.9 with step 0.2.

All the experiments are conducted by Matlab2016a on a work station with Xeon CPU E3-1245 v3 (3.4GHz) and 32.0 GB RAM memory on the Windows 7 operating system.
The codes of all the compared methods are downloaded from the authors' homepages or provided by their authors\footnote{The implementation of RPCA is obtained from the TFOCS, a Matlab toolbox available at https://github.com/cvxr/TFOCS.}.

The clustering performance are evaluated by four metrics, the normalized mutual information (NMI), the clustering accuracy (ACC), the F score, and the adjusted rand index (AdjRI).
All the compared methods need to perform K-Means to get the final partition. To avoid inaccuracy brought by random initialization, K-Means is repeated 50 times and the average score and
standard deviation are reported.

\subsection{Subspace Recovery Results on Synthetic Data}
The ability of LCRSR to recover the latent complete row space is tested here on the synthetic data.
For comparison, RPCA is also performed to recover  $\mathbf{V}_0\mathbf{V}_0^T$: estimating $\mathbf{G}\mathbf{L}_0$ by solving problem (\ref{org3}) and then performing SVD on the estimated $\mathbf{G}\mathbf{L}_0$ to obtain $\mathbf{V}_0\mathbf{V}_0^T$.

We generate data matrices of two views, according to the model $\mathbf{X}^{(v)} = \mathbf{G}^{(v)}\mathbf{L}_0 + \mathbf{S}_0^{(v)}$ ($v = 1, 2$).
We consider the case that $d^{(1)} = d^{(2)} = 100$, $m = 50$, and $n = 200$.
Concretely, we construct $\mathbf{L}_0 \in \mathbb{R}^{m\times n}$ by sampling 100 points from each of two randomly generated subspaces.
Note $\mathbf{G}^{(v)}$ is required to have orthogonal columns when using RPCA to recover the latent row space.
To make fair comparisons, we randomly generate
$\mathbf{G}^{(v)}\in \mathbb{R}^{d^{(v)}\times m}$ ($v = 1, 2$) to make it have orthogonal columns.
$\mathbf{S}_0^{(v)}$ ($v = 1, 2$) is composed of $\{-1, 0, 1\}$ with a Bernoulli model, i.e., $\mathbf{S}_{0,ij}^{(v)} $ equals to 0 with probability $1-\rho_s$, and $\pm1$ each with probability $\rho_s/2$.
The values of $r$ and $\rho_s$ are set in the range of $[1:1:20]$ and $[0.02:0.02:0.4]$, respectively.
Thus, $rank(\mathbf{L}_0)$ varies from 2 to 40 with an interval of 2.

The Signal-to-Noise Ratio ($\text{SNR}_{\text{dB}}$) is employed to evaluate the recovery quality, i.e., the similarity between  $\mathbf{V}_0\mathbf{V}_0^T$ and $\hat{\mathbf{V}}_0\hat{\mathbf{V}}_0^T$.
We define the recovery score for each  $(r, \rho_s)$ pair as
\begin{equation}
    \text{score} = \begin{cases}
    0, & \text{SNR}_{\text{dB}} < 15,\\
    0.2 & 15\leq \text{SNR}_{\text{dB}} < 20,\\
    0.5 & 20\leq \text{SNR}_{\text{dB}} < 30,\\
    1, & \text{SNR}_{\text{dB}} > 30.
    \end{cases}
\end{equation}

The rank $r_0$ is assumed to be given when running LCRSR.
For each $(r, \rho_s)$ pair, experiments are repeated on 20 random simulations, and the average score is recorded.
Fig. \ref{figVrecover} shows the recovery results, i.e., the average score from 20 random trials.
As we can see, though $\mathbf{G}^{(v)} (\forall v)$ has orthogonal columns, the recovery performance of RPCA is still limited and it makes no correct recovery.
In comparison, the proposed LCRSR succeeds when $r_0$ is relatively low and $\{\mathbf{S}_0^{(v)}\}_{v=1}^V$ is relatively sparse.
In a word, LCRSR performs better than RPCA in recovering the latent complete row space.
\subsection{Clustering Performance}
We consider to test the clustering performance of LCRSR on the first 6 datasets.
The input matrices are normalized to have unit $\ell_2$ norm column-wisely.
The clustering comparison results are displayed in Table \ref{Table1}.
Besides, the corresponding Wilcoxon signed-rank test results with confidence level 0.01 are also reported.

From the results, the following observations can be obtained.
(1) Feature concatenation is likely to fail on datasets with heterogeneous features.
For example, on Wiki and Prok, ConLRR/ConRPCA performs worse than BestLRR/BestRPCA, though information from multiple views are used.
Note that both Wiki and Prok are actually cross-modal, i.e., text-image and text-gene respectively.
Thus, it could be inferred that the direct concatenation of heterogeneous features might deteriorate the performance.
(2) Though both LMSC and the proposed LCRSR are based on the latent representation assumption, LCRSR shows advantages in clustering performance over LMSC.
The reason might be two-fold.
First, LCRSR collects complete information from multiple views while LMSC cannot guarantee the information completeness of the learned latent representation.
Thus, LCRSR shows superiority with the better maintained underlying clustering structure.
Second, the model of LMSC is actually formulated on the concatenation of multiple views, which will somehow impair its performance.
(3) While the performance of some of the compared multi-view subspace clustering approaches is unstable, LCRSR achieves robust clustering performance in spite of varying data types.
For example, CSMSC ranks the second on both Derm and Forest, but it does not perform very well on Wiki.
The possible reason is following.
The compared multi-view competitors learn self-representation matrices view by view, the underly clustering structure could be damaged when forming the final affinity matrix from view-specific self-representations.
In contrast, LCRSR directly recovers the authentic latent complete row space of the data, with the underly clustering structure being better protected.

On the whole, the clustering performance of the proposed LCRSR outperforms the baselines in most cases in terms of all the four metrics.

\subsection{Background Subtraction}
Multi-camera video surveillance is a natural application of multi-view learning.
Background modeling from video clips shot by multiple static cameras can be taken as a multi-view low-rank matrix analysis problem.
Concretely, the low-rank part corresponds to the background, while the sparse component corresponds to moving objects \cite{Liu2018RSP,candes2011RPCA}.
Thus, in this subsection, we apply the proposed LCRSR to background subtraction from multi-camera video surveillance.

The experiments are conducted on the Lab and DTHC video datasets.
To visualize the sparse components, different from the above clustering experiments, the data matrices are not normalized.
For comparison, we also implement RPCA \cite{candes2011RPCA} on each single view.
For each view, the parameter $\lambda$ for RPCA is well tuned around $1/\sqrt{(\max(n, d^{(v)}))}$, as suggested in \cite{candes2011RPCA}.
Fig. \ref{Fig.Moving1} and Fig. \ref{Fig.Moving2} present frames at two time points from each of the video datasets.
For Lab, the proposed LCRSR successfully detaches the moving objects from the background, while the background extracted by RPCA still contains a portion of the objects' pixels.
Since DTHC is captured from real scenarios, which records a group of people disperse from the center quickly, the task becomes more challenging.
As shown in Fig. \ref{Fig.Moving2}, both methods fail to separate the moving objects from the background completely.
In comparison, LCRSR only confuses the slow-moving people with the background and performs much better than RPCA.

\begin{figure}
\centering
\subfigure[]{
\includegraphics[width=0.23\textwidth]{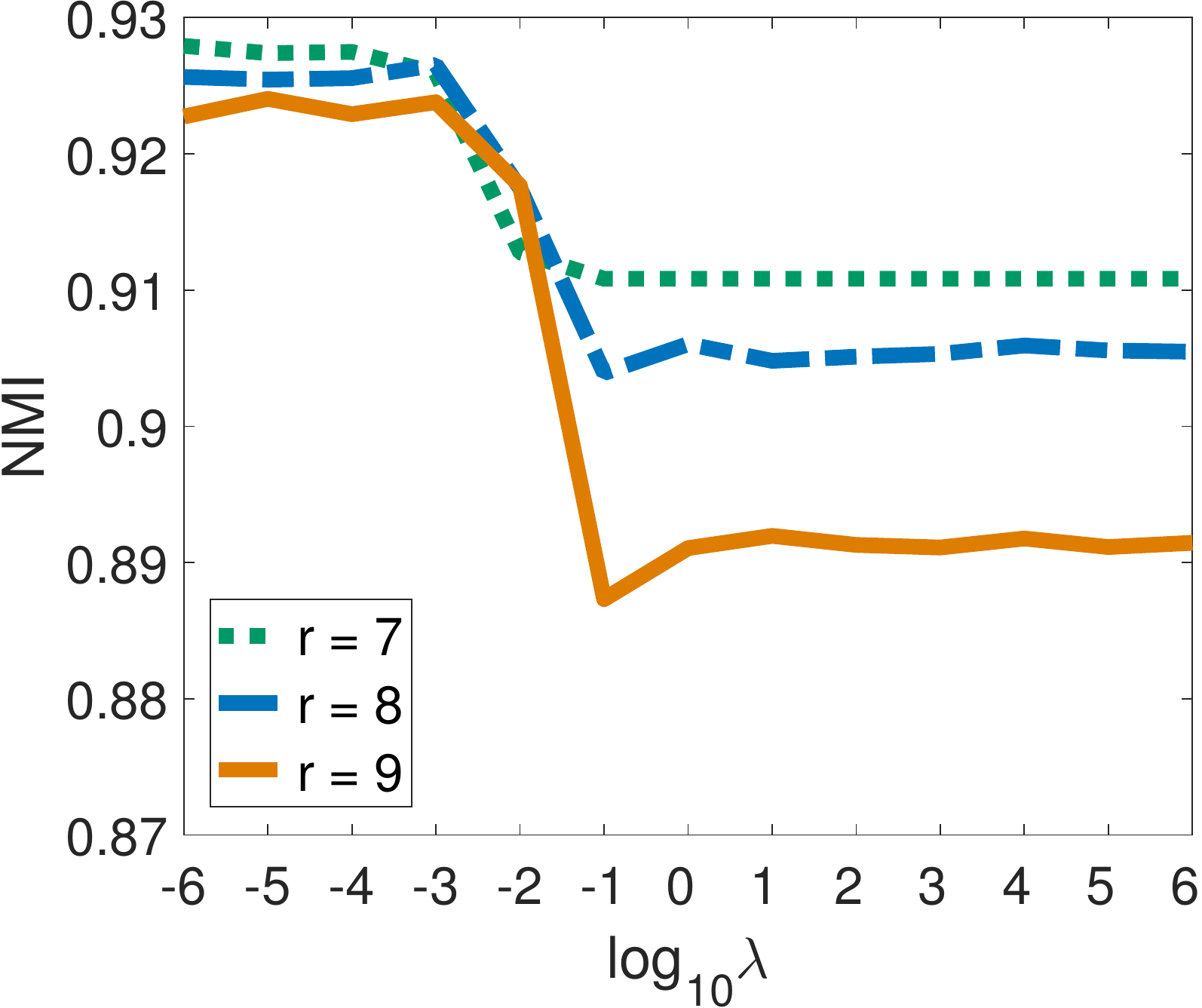}\label{Fig.Para_sub1}}
\subfigure[]{
\includegraphics[width=0.23\textwidth]{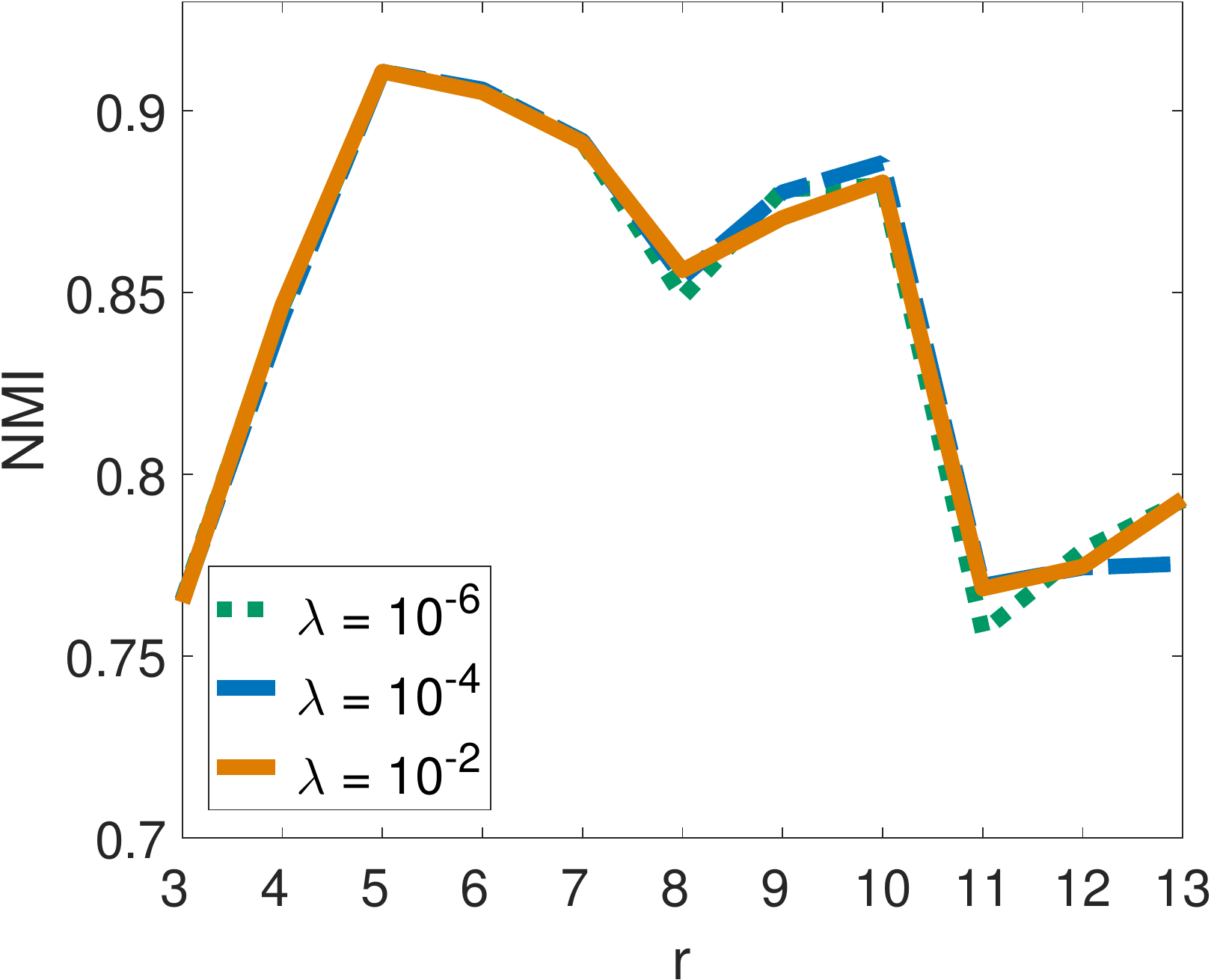}\label{Fig.Para_sub2}}
\centering
\caption{Sensitivity of LCRSR w.r.t. parameters ($\lambda$ and $r$) evaluated by NMI on the Derm dataset. (a) The parameter $\lambda$ varies when $r$ is fixed as 7, 8 and 9, respectively. (b) The parameter $r$ varies when $\lambda$ is set as $10^{-6}$, $10^{-4}$, and $10^{-2}$, respectively.}
\label{Fig.Para}
\end{figure}

\begin{figure}
\centering
\subfigure[]{
\includegraphics[width=0.225\textwidth]{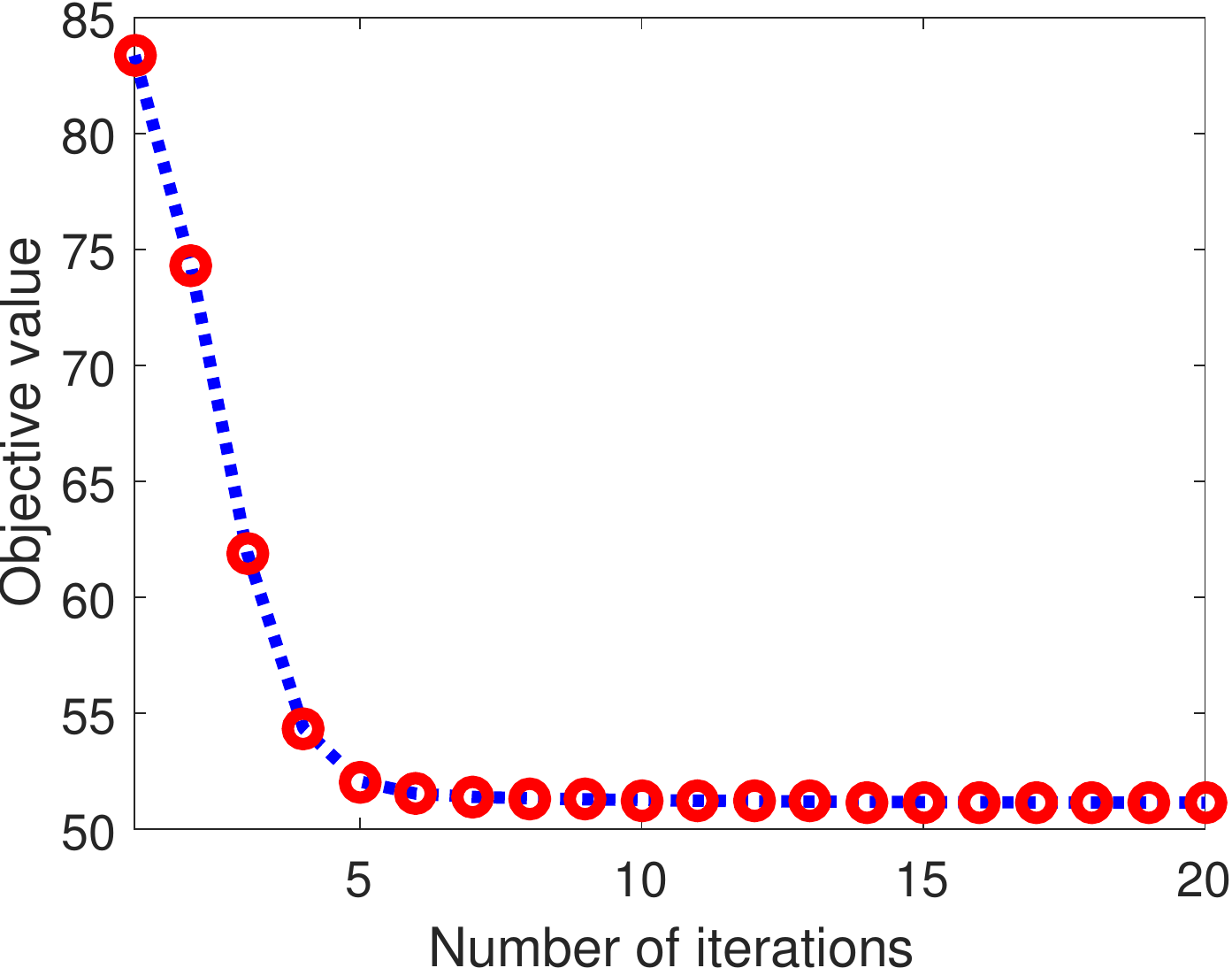}}
\subfigure[]{
\includegraphics[width=0.23\textwidth]{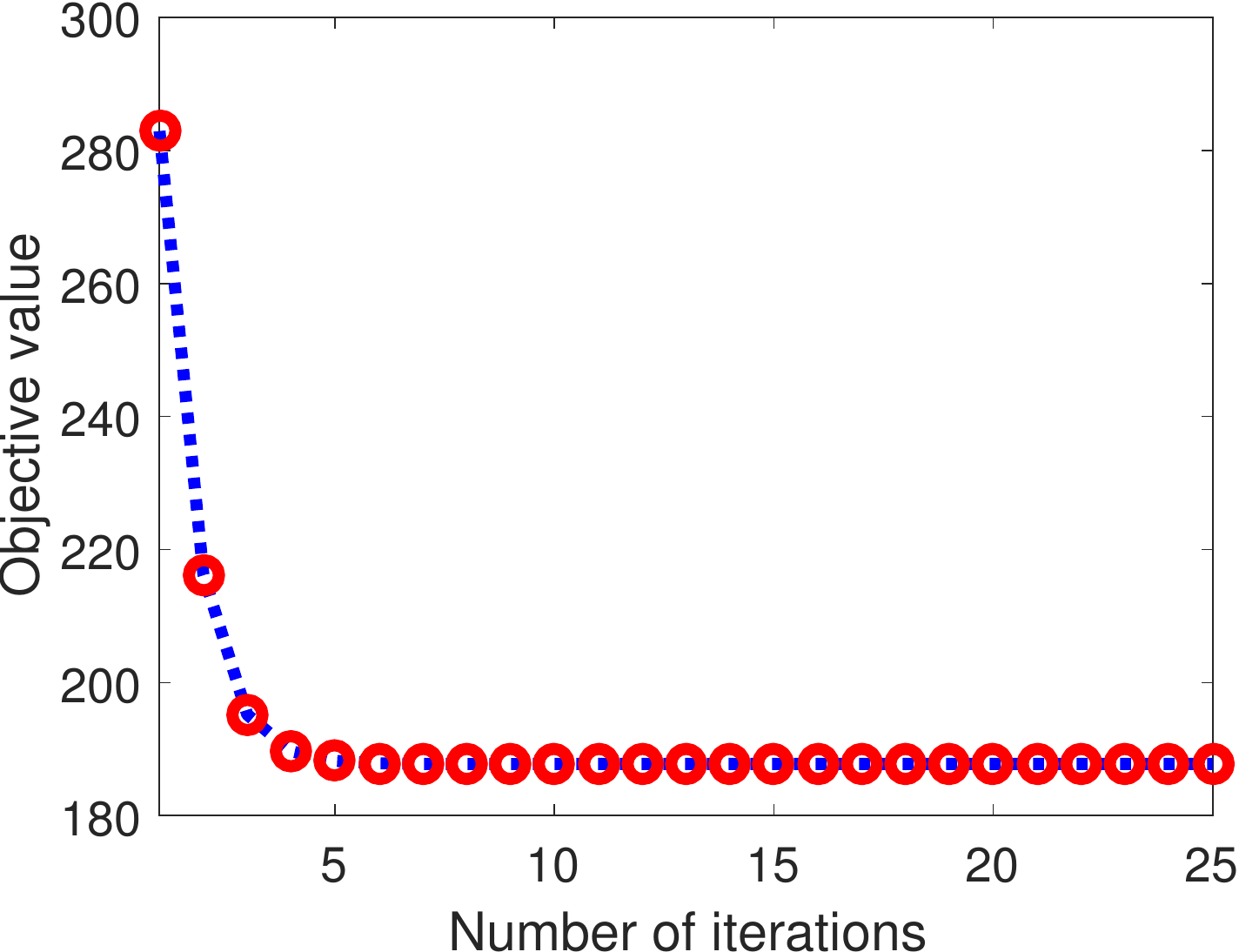}}
\centering
\caption{Convergence performance of LCRSR on two datasets: (a) Wiki. (b) USPS.}
\label{Fig.Conv}
\end{figure}

\subsection{Parameter Study}
In this subsection, we study the sensitivity of LCRSR with respect to parameters.
There are two parameters in LCRSR, i.e., the low-rank parameter $r$ and the trade-off parameter $\lambda$.
Experiments are conducted on the Derm dataset, and the experimental results are shown in Fig. \ref{Fig.Para}.
Since the results are similar for the four evaluation metrics, only the NMI results are reported.

As shown in Fig. \ref{Fig.Para_sub1}, the performance of LCRSR degenerates dramatically when $\lambda \geq 10^{-2}$.
This is because that the solution to $\mathbf{S}_0^{(v)}$ will become all zeros when $\lambda$ is sufficiently large.
As for the rank parameter $r$, Fig. \ref{Fig.Para_sub2} shows that LCRSR could perform well when $r$ is in a certain range.
For this dataset, $r = 5$ is a good choice.

\subsection{Convergence and Computational Time Comparison}
The convergence behavior and computational time of LCRSR are considered in this subsection.
We regard the algorithms to be converged if the relative difference of objectives between two successive iterations is less than $10^{-4}$.
Fig. \ref{Fig.Conv} show the plots of objective values of LCRSR on datasets Wiki and USPS.
From the plots, we can see that LCRSR converges in about 20 iterations.

Table \ref{Table2} summarizes the computational complexity of the compared multi-view subspace clustering methods, where $T$ denotes the number of iterations needed to converge for the corresponding method.
Except for LCRSR, all of the rest methods have a cubic time complexity of $n$.
As for the dimension $d$, except for LMSC, the time complexity of all the other method is linear to $d$.
Accordingly, Table \ref{Table3} displays the time comparison results.
Table \ref{Table3} shows that MVSC spends the most time for all datasets.
This might be because MVSC has the most complex constraints and requires more iterations to converge.
Consistent with the computational complexity analysis, the proposed LCRSR costs the least time on all datasets.

\begin{table}
  \centering
  \caption{Time complexity of the compared multi-view subspace clustering methods. }\label{Table2}
 \begin{tabular}{|c|c|} \hline
 Methods &  Time complexity \\ \hline
 LMSC    &  $\mathcal{O}(T(n^3 + d^3))$ \\ \hline
 MVSC    &  $\mathcal{O}(Tn^2(n + d))$ \\\hline
 CSMSC   &  $\mathcal{O}(TVn(n^2 + d))$ \\\hline
 LTMSC   &  $\mathcal{O}(Tn^2(n+d))$    \\\hline
 MLRSSC  &  $\mathcal{O}(TVn^2(n+d))$        \\\hline
 LCRSR   &  $\mathcal{O}(Tn^2(d + r))$  \\\hline
 \end{tabular}
\end{table}

\begin{table}
  \centering
  \caption{Computational time comparison (seconds). }\label{Table3}
 \begin{tabular}{|c|c|c|c|c|c|c|} \hline
 Methods &  Derm  & Forest & Prok   & Wiki      & Caltech7   & USPS\\ \hline
 LMSC    &  6.65  & 14.09  & 18.18  & 1681.93   & 11.98     &5294.40 \\ \hline
 MVSC    &  21.98 & 46.85  & 255.51 & 6881.54   & 247.44  & 62584.69\\\hline
 CSMSC   &  1.22  & 1.03   & 1.56   & 28.67     & 1.79    & 506.52\\\hline
 LTMSC   &  6.43  & 14.15  & 28.47  & 1487.29   & 21.11   & 4401.68\\\hline
 MLRSSC   &  -  & -  & - & -  & -   & -\\\hline
 LCRSR   &  \textbf{0.11}  & \textbf{0.12}   & \textbf{0.44}   & \textbf{4.73}      & \textbf{0.44}    &\textbf{ 78.34}\\\hline
 \end{tabular}
\end{table}

\section{Conclusions}\label{sec_summary}
With the goal of more accurate and faster multi-view subspace clustering, in this paper, the LCRSR method is proposed based on the latent representation assumption.
By integrating complete information from multiple views, LCRSR is able to directly recover both the latent complete row space and the sparse errors existing in the original multi-view data.
LCRSR is more scalable to large-scale datasets since it technically avoids the computationally expensive graph construction procedure and can be solved efficiently.
The effectiveness and efficiency of LCRSR are demonstrated by extensive experiments on various multi-view datasets.


\bibliographystyle{IEEEtran}

\begin{thebibliography}{10}
\providecommand{\url}[1]{#1}
\csname url@samestyle\endcsname
\providecommand{\newblock}{\relax}
\providecommand{\bibinfo}[2]{#2}
\providecommand{\BIBentrySTDinterwordspacing}{\spaceskip=0pt\relax}
\providecommand{\BIBentryALTinterwordstretchfactor}{4}
\providecommand{\BIBentryALTinterwordspacing}{\spaceskip=\fontdimen2\font plus
\BIBentryALTinterwordstretchfactor\fontdimen3\font minus
  \fontdimen4\font\relax}
\providecommand{\BIBforeignlanguage}[2]{{%
\expandafter\ifx\csname l@#1\endcsname\relax
\typeout{** WARNING: IEEEtran.bst: No hyphenation pattern has been}%
\typeout{** loaded for the language `#1'. Using the pattern for}%
\typeout{** the default language instead.}%
\else
\language=\csname l@#1\endcsname
\fi
#2}}
\providecommand{\BIBdecl}{\relax}
\BIBdecl

\bibitem{blum1998cotraining}
A.~Blum and T.~Mitchell, ``Combining labeled and unlabeled data with
  co-training,'' in \emph{COLT}, 1998, pp. 92--100.

\bibitem{Bickel2004MvC}
S.~Bickel and T.~Scheffer, ``Multi-view clustering,'' in \emph{ICDM}, 2004, pp.
  19--26.

\bibitem{White2012CMSL}
M.~White, Y.~Yu, X.~Zhang, and D.~Schuurmans, ``Convex multi-view subspace
  learning,'' in \emph{NIPS}, 2012, pp. 1673--1681.

\bibitem{xu2013MVLsurvey}
\BIBentryALTinterwordspacing
C.~Xu, D.~Tao, and C.~Xu, ``A survey on multi-view learning,'' \emph{CoRR},
  vol. abs/1304.5634, 2013. [Online]. Available:
  \url{http://arxiv.org/abs/1304.5634}
\BIBentrySTDinterwordspacing

\bibitem{Guo2013CSRL}
Y.~Guo, ``Convex subspace representation learning from multi-view data,'' in
  \emph{IJCAI}, 2013, pp. 387--393.

\bibitem{Xu2015Intact}
C.~{Xu}, D.~{Tao}, and C.~{Xu}, ``Multi-view intact space learning,''
  \emph{IEEE Transactions on Pattern Analysis and Machine Intelligence},
  vol.~37, no.~12, pp. 2531--2544, Dec 2015.

\bibitem{Wang2015consensus}
Y.~{Wang}, X.~{Lin}, L.~{Wu}, W.~{Zhang}, Q.~{Zhang}, and X.~{Huang}, ``Robust
  subspace clustering for multi-view data by exploiting correlation
  consensus,'' \emph{IEEE Transactions on Image Processing}, vol.~24, no.~11,
  pp. 3939--3949, Nov 2015.

\bibitem{Hou2017MUFE}
C.~Hou, F.~Nie, H.~Tao, and D.~Yi, ``Multi-view unsupervised feature selection
  with adaptive similarity and view weight,'' \emph{IEEE Transactions on
  Knowledge and Data Engineering}, vol.~29, no.~9, pp. 1998--2011, Sept 2017.

\bibitem{2017WangLapLRR}
B.~Wang, Y.~Hu, J.~Gao, Y.~Sun, and B.~Yin, ``Laplacian \protect{LRR} on
  product grassmann manifolds for human activity clustering in multicamera
  video surveillance,'' \emph{IEEE Transactions on Circuits and Systems for
  Video Technology}, vol.~27, no.~3, pp. 554--566, March 2017.

\bibitem{Zhang2018gLMSC}
C.~Zhang, H.~Fu, Q.~Hu, X.~Cao, Y.~Xie, D.~Tao, and D.~Xu, ``Generalized latent
  multi-view subspace clustering,'' \emph{IEEE Transactions on Pattern Analysis
  and Machine Intelligence}, pp. 1--14, 2018.

\bibitem{Cheng2019Tensor}
M.~{Cheng}, L.~{Jing}, and M.~K. {Ng}, ``Tensor-based low-dimensional
  representation learning for multi-view clustering,'' \emph{IEEE Transactions
  on Image Processing}, vol.~28, no.~5, pp. 2399--2414, May 2019.

\bibitem{Wu2019essentialTensor}
J.~{Wu}, Z.~{Lin}, and H.~{Zha}, ``Essential tensor learning for multi-view
  spectral clustering,'' \emph{IEEE Transactions on Image Processing}, vol.~28,
  no.~12, pp. 5910--5922, Dec 2019.

\bibitem{zhang2015LTMSC}
C.~Zhang, H.~Fu, S.~Liu, G.~Liu, and X.~Cao, ``Low-rank tensor constrained
  multiview subspace clustering,'' in \emph{ICCV}, 2015, pp. 1582--1590.

\bibitem{cao2015DiMSC}
X.~Cao, C.~Zhang, H.~Fu, S.~Liu, and H.~Zhang, ``Diversity-induced multi-view
  subspace clustering,'' in \emph{CVPR}, 2015, pp. 586--594.

\bibitem{Ding2016RMSL}
Z.~Ding and Y.~Fu, ``Robust multi-view subspace learning through dual low-rank
  decompositions,'' in \emph{AAAI}, 2016, pp. 1181--1187.

\bibitem{Gao2015MVSC}
H.~Gao, F.~Nie, X.~Li, and H.~Huang, ``Multi-view subspace clustering,'' in
  \emph{ICCV}, 2015, pp. 4238--4246.

\bibitem{Luo2018CSMSC}
S.~Luo, C.~Zhang, W.~Zhang, and X.~Cao, ``Consistent and specific multi-view
  subspace clustering,'' in \emph{AAAI}, 2018, pp. 3730--3737.

\bibitem{Yin2018SCMV3D}
M.~{Yin}, J.~{Gao}, S.~{Xie}, and Y.~{Guo}, ``Multiview subspace clustering via
  tensorial t-product representation,'' \emph{IEEE Transactions on Neural
  Networks and Learning Systems}, vol.~30, no.~3, pp. 851--864, March 2019.

\bibitem{wang2017ECRMSC}
X.~Wang, X.~Guo, Z.~Lei, C.~Zhang, and S.~Z. Li, ``Exclusivity-consistency
  regularized multi-view subspace clustering,'' in \emph{CVPR}, 2017, pp.
  923--931.

\bibitem{Zhang2017LMSC}
C.~Zhang, Q.~Hu, H.~Fu, P.~Zhu, and X.~Cao, ``Latent multi-view subspace
  clustering,'' in \emph{CVPR}, 2017, pp. 4333--4341.

\bibitem{Liu2013LRR}
G.~Liu, Z.~Lin, S.~Yan, J.~Sun, Y.~Yu, and Y.~Ma, ``Robust recovery of subspace
  structures by low-rank representation,'' \emph{IEEE Transactions on Pattern
  Analysis and Machine Intelligence}, vol.~35, no.~1, pp. 171--184, Jan 2013.

\bibitem{2013SSC}
E.~Elhamifar and R.~Vidal, ``Sparse subspace clustering: Algorithm, theory, and
  applications,'' \emph{IEEE Transactions on Pattern Analysis and Machine
  Intelligence}, vol.~35, no.~11, pp. 2765--2781, Nov 2013.

\bibitem{Lu2019BDR}
C.~{Lu}, J.~{Feng}, Z.~{Lin}, T.~{Mei}, and S.~{Yan}, ``Subspace clustering by
  block diagonal representation,'' \emph{IEEE Transactions on Pattern Analysis
  and Machine Intelligence}, vol.~41, no.~2, pp. 487--501, Feb 2019.

\bibitem{shi2000Ncut}
J.~Shi and J.~Malik, ``Normalized cuts and image segmentation,'' \emph{IEEE
  Transactions on Pattern Analysis and Machine Intelligence}, vol.~22, no.~8,
  pp. 888--905, 2000.

\bibitem{Liu2018RSP}
G.~{Liu}, Z.~{Zhang}, Q.~{Liu}, and H.~{Xiong}, ``Robust subspace clustering
  with compressed data,'' \emph{IEEE Transactions on Image Processing},
  vol.~28, no.~10, pp. 5161--5170, Oct 2019.

\bibitem{gretton2005HSIC}
A.~Gretton, O.~Bousquet, A.~Smola, and B.~Scholkopf, ``Measuring statistical
  dependence with \protect{Hilbert-Schmidt} norms,'' in \emph{ALT},
  vol.~16.\hskip 1em plus 0.5em minus 0.4em\relax Springer, 2005, pp. 63--78.

\bibitem{brbic2018MLRSSC}
M.~Brbi\'c and I.~Kopriva, ``Multi-view low-rank sparse subspace clustering,''
  \emph{Pattern Recognition}, vol.~73, pp. 247--258, 2018.

\bibitem{candes2011RPCA}
E.~J. Cand{\`e}s, X.~Li, Y.~Ma, and J.~Wright, ``Robust principal component
  analysis?'' \emph{Journal of the ACM}, vol.~58, no.~3, p.~11, 2011.

\bibitem{Costeira1998SIM}
J.~P. Costeira and T.~Kanade, ``A multibody factorization method for
  independently moving objects,'' \emph{International Journal of Computer
  Vision}, vol.~29, no.~3, pp. 159--179, 1998.

\bibitem{NIPS2013_SSCLRR}
Y.-X. Wang, H.~Xu, and C.~Leng, ``Provable subspace clustering: When
  \protect{LRR} meets \protect{SSC},'' in \emph{NIPS}, 2013, pp. 64--72.

\bibitem{motion4d}
C.~Tomasi and T.~Kanade, ``Shape and motion from image streams under
  orthography,'' \emph{International Journal of Computer Vision}, vol.~9,
  no.~2, pp. 137--154, 1992.

\bibitem{face9d}
R.~Basri and D.~Jacobs, ``Lambertian reflection and linear subspaces,'' in
  \emph{ICCV}, 2001, pp. 383--390.

\bibitem{Candes2005RIP}
E.~J. Cand{\`e}s and T.~Tao, ``Decoding by linear programming,'' \emph{IEEE
  Transactions on Information Theory}, vol.~51, no.~12, pp. 4203--4215, 2005.

\bibitem{Attouch2009ConProx}
H.~Attouch and J.~Bolte, ``On the convergence of the proximal algorithm for
  nonsmooth functions involving analytic features,'' \emph{Mathematical
  Programming}, vol. 116, no. 1-2, pp. 5--16, 2009.

\bibitem{Parikh2014proximal}
N.~Parikh and S.~Boyd, ``Proximal algorithms,'' \emph{Foundations and Trends in
  Optimization}, vol.~1, no.~3, pp. 127--239, 2014.

\bibitem{Beck2009FISTA}
A.~Beck and M.~Teboulle, ``A fast iterative shrinkage-thresholding algorithm
  for linear inverse problems,'' \emph{SIAM J. IMAGING SCIENCES}, vol.~2,
  no.~1, pp. 183--202, 2009.

\bibitem{Golub2013MC}
G.~Golub and C.~V. Loan, \emph{Matrix Computations}, 4th~ed., ser. Johns
  Hopkins Studies in the Mathematical Sciences.\hskip 1em plus 0.5em minus
  0.4em\relax Johns Hopkins University Press, 2013, ch.~8, pp. 391--469.

\bibitem{johnson2012ForestData}
B.~Johnson, R.~Tateishi, and Z.~Xie, ``Using geographically weighted variables
  for image classification,'' \emph{Remote Sensing Letters}, vol.~3, no.~6, pp.
  491--499, 2012.

\bibitem{brbic2016landscape}
M.~Brbi{\'c}, M.~Pi{\v{s}}korec, V.~Vidulin, A.~Kri{\v{s}}ko, T.~{\v{S}}muc,
  and F.~Supek, ``The landscape of microbial phenotypic traits and associated
  genes,'' \emph{Nucleic Acids Research}, vol.~44, no.~21, pp.
  10\,074--10\,090, 2016.

\bibitem{Salton1988TFIDF}
G.~Salton and C.~Buckley, ``Term-weighting approaches in automatic text
  retrieval,'' \emph{Information Processing \& Management}, vol.~24, no.~5, pp.
  513--523, 1988.

\bibitem{lowe2004SIFT}
D.~G. Lowe, ``Distinctive image features from scale-invariant keypoints,''
  \emph{International Journal of Computer Vision}, vol.~60, no.~2, pp. 91--110,
  2004.

\bibitem{oliva2001GIST}
A.~Oliva and A.~Torralba, ``Modeling the shape of the scene: A holistic
  representation of the spatial envelope,'' \emph{International Journal of
  Computer Vision}, vol.~42, no.~3, pp. 145--175, 2001.

\bibitem{dalal2005HOG}
N.~Dalal and B.~Triggs, ``Histograms of oriented gradients for human
  detection,'' in \emph{CVPR}, vol.~1, 2005, pp. 886--893.

\bibitem{2011LabData}
J.~Berclaz, F.~Fleuret, E.~Turetken, and P.~Fua, ``Multiple object tracking
  using k-shortest paths optimization,'' \emph{IEEE Transactions on Pattern
  Analysis and Machine Intelligence}, vol.~33, no.~9, pp. 1806--1819, Sept
  2011.

\end{thebibliography}

\ifCLASSOPTIONcaptionsoff
  \newpage
\fi

\end{document}